%% file: main.tex
\documentclass[sigconf]{acmart}

\AtBeginDocument{%
  \providecommand\BibTeX{{%
    \normalfont B\kern-0.5em{\scshape i\kern-0.25em b}\kern-0.8em\TeX}}}

\usepackage{indentfirst}
\usepackage{color}
\usepackage{enumitem}
\usepackage{multirow}
\usepackage{colortbl}
\usepackage{tabularray}
\usepackage{graphicx}
\usepackage{subfigure}
\usepackage{graphicx, nicefrac}
\usepackage{url}

\usepackage{breakurl}
\usepackage[flushleft]{threeparttable}
\usepackage{bigstrut}
\usepackage{ifthen}

\newtheorem{theorem}{Theorem}
\newtheorem{lemma}{Lemma}[section]

\newtheorem{definition}{Definition}

\newtheorem{assumption}{Assumption}
\usepackage[ruled,linesnumbered]{algorithm2e}
\input{math_command}

\newif\ifkdd

\kddfalse

\setcopyright{acmcopyright}
\copyrightyear{2024}
\acmYear{2024}
\setcopyright{acmlicensed}
\acmConference[KDD '24] {Proceedings of the 30th ACM SIGKDD Conference on Knowledge Discovery and Data Mining }{August 25--29, 2024}{Barcelona, Spain.}
\acmBooktitle{Proceedings of the 30th ACM SIGKDD Conference on Knowledge Discovery and Data Mining (KDD '24), August 25--29, 2024, Barcelona, Spain}
\acmISBN{979-8-4007-0490-1/24/08}
\acmDOI{10.1145/3637528.3671926}

\newcommand{\stitle}[1]{\vspace{2mm} \noindent {\bf #1}}

\begin{document}

\title{Calibration of Time-Series Forecasting:
\\Detecting and Adapting Context-Driven Distribution Shift}

\renewcommand{\shorttitle}{Calibration of Time-Series Forecasting: Detecting and Adapting Context-Driven Distribution Shift}

\author{Mouxiang Chen}
\authornote{Both authors contributed equally to this research.}
\affiliation{
  \institution{Zhejiang University}
  \city{Hangzhou}
  \country{China}}
\email{chenmx@zju.edu.cn}

\author{Lefei Shen}
\authornotemark[1]
\affiliation{
  \institution{Zhejiang University}
  \city{Hangzhou}
  \country{China}}
\email{lefeishen@zju.edu.cn}

\author{Han Fu}
\affiliation{
  \institution{Zhejiang University}
  \city{Hangzhou}
  \country{China}}
\email{11821003@zju.edu.cn}

\author{Zhuo Li}
\authornote{Corresponding authors.}
\affiliation{
  \institution{State Street Technology \\ (Zhejiang) Ltd.}
  \city{Hangzhou}
  \country{China}}
\email{lizhuo@zju.edu.cn}

\author{Jianling Sun}
\affiliation{
  \institution{Zhejiang University}
  \city{Hangzhou}
  \country{China}}
\email{sunjl@zju.edu.cn}

\author{Chenghao Liu}
\authornotemark[2]
\affiliation{
  \institution{Salesforce Research Asia}
  \country{Singapore}}
\email{chenghao.liu@salesforce.com}

\newcommand{\eg}{{\it e.g.}}
\newcommand{\etal}{{\it et al.}}
\newcommand{\ie}{{\it i.e.}}
\newcommand\todo[1]{\textcolor{blue}{#1}}
\newcommand\method[1]{\textsf{#1}}

\newcommand\remove[1]{\textcolor{blue}{(#1)}}
\newcommand\add[1]{\textcolor{red}{#1}}

\begin{abstract}
Recent years have witnessed the success of introducing deep learning models to time series forecasting. From a data generation perspective, we illustrate that existing models are susceptible to distribution shifts driven by temporal contexts, whether observed or unobserved. Such context-driven distribution shift (CDS) introduces biases in predictions within specific contexts and poses challenges for conventional training paradigms. In this paper, we introduce a universal calibration methodology for the detection and adaptation of CDS with a trained model. To this end, we propose a novel CDS detector, termed the "residual-based CDS detector" or "\method{Reconditionor}", which quantifies the model's vulnerability to CDS by evaluating the mutual information between prediction residuals and their corresponding contexts. A high \method{Reconditionor} score indicates a severe susceptibility, thereby necessitating model adaptation. In this circumstance, we put forth a straightforward yet potent adapter framework for model calibration, termed the "sample-level contextualized adapter" or "\method{SOLID}". This framework involves the curation of a contextually similar dataset to the provided test sample and the subsequent fine-tuning of the model's prediction layer with a limited number of steps. Our theoretical analysis demonstrates that this adaptation strategy can achieve an optimal bias-variance trade-off. Notably, our proposed \method{Reconditionor} and \method{SOLID} are model-agnostic and readily adaptable to a wide range of models. Extensive experiments show that \method{SOLID} consistently enhances the performance of current forecasting models on real-world datasets, especially on cases with substantial CDS detected by the proposed \method{Reconditionor}, thus validating the effectiveness of the calibration approach. 
\end{abstract}



\maketitle

\input{sec-introduction}

\input{sec-related_works}

\input{sec-preliminaries}

\input{sec-methodology}

\input{sec-experiments}

\input{sec-conclusion}

\bibliographystyle{ACM-Reference-Format}
\bibliography{reference-kdd}

\ifkdd
\input{sec-appendix-kdd}
\else
\input{sec-appendix-full}

\fi

\end{document}

%% file: math_command.tex
\usepackage{amsmath,amsfonts,bm}

\def\sref#1{\S~\ref{#1}}

\def\eqref#1{equation~\ref{#1}}
\def\Eqref#1{Eq.(\ref{#1})}

\def\1{\bm{1}}

\def\vx{{\bm{x}}}

\def\mX{{\bm{X}}}

\DeclareMathAlphabet{\mathsfit}{\encodingdefault}{\sfdefault}{m}{sl}
\SetMathAlphabet{\mathsfit}{bold}{\encodingdefault}{\sfdefault}{bx}{n}



%% file: sec-introduction.tex
\section{Introduction}
Time Series Forecasting (TSF) plays a pivotal role in numerous real-world applications, including energy consumption planning \cite{TSF_Energy_1, TSF_Energy_2}, weather forecasting \cite{TSF_Weather_1, TSF_Weather_2}, financial risk assessment \cite{TSF_Finance_1, TSF_Economics_1}, and web recommendation \cite{TSF_Web_1, TSF_Web_2, TSF_Web_3}.
Recent years have witnessed the progress of introducing time series forecasting model \cite{Informer, Autoformer, FEDformer, ETSformer, Crossformer} to better capture the temporal dependencies by extracting and stacking multi-level features. 
Despite the remarkable architecture design, the distribution shift \cite{distribution-shift} has become an unavoidable yet highly challenging issue, which engenders suboptimal performance and hampers generalization with a fluctuating distribution.

\begin{figure}[t]
    \centering
    \subfigure[causal graph]{
        \label{fig:intro-causal}
        \includegraphics[width=0.22\textwidth]{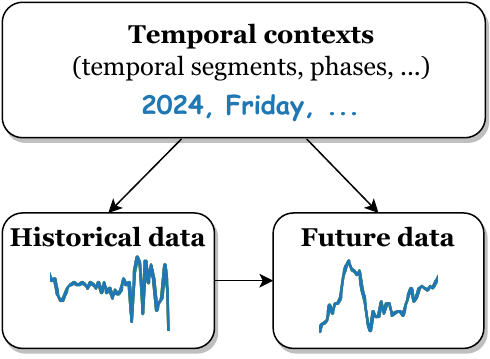}
    }
    \subfigure[(conditional) residual distribution]{
        \label{fig:intro-residuals}
        \includegraphics[width=0.22\textwidth,trim={10 10 10 10},clip]{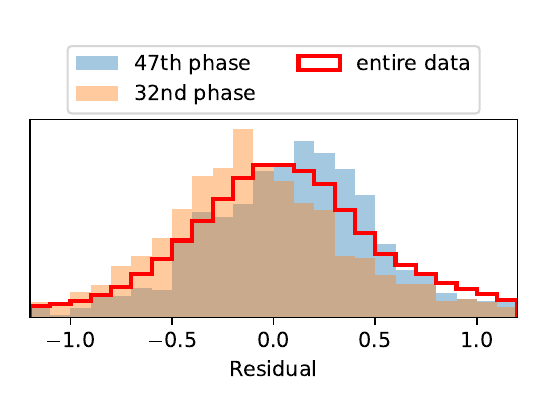}
    }
    \caption{(a) Causal graph in the presence of context-driven distribution shift. (b) Impact of CDS: Autoformer's residual distribution on the entire Illness dataset, along with the residual distributions conditioned on two different periodic phases. }
    \label{fig:intro}
    \vspace{-0.4cm}
\end{figure}

Generally, distribution shift signifies variations in the underlying data generation process, which is typically driven by some temporal observed or unobserved factors, namely \textit{contexts}. In this paper, we reveal two significant observed contexts within time series data: \textbf{temporal segments} (\ie, different temporal stages as the time evolution),  and \textbf{periodic phases} (\ie, the fraction of the period covered up to the current time), along with other \textbf{unobserved contexts}. 
For example, in the scenario of electricity consumption, factors like economic trends over the years (temporal segments) and seasonal fluctuations (periodic phases) can affect electricity usage. We visualized the impact of these two contexts on data distribution in Appendix \ref{sec:appendix_context_visual}.
Moreover, sudden policy changes (unobserved contexts) can also affect the usage. We refer to this phenomenon as \textit{context-driven distribution shift}, or \textit{CDS}.

In the presence of CDS, TSF models remain constrained owing to their ignorance of contexts. Firstly, the training and testing datasets are often generated under distinct contexts (\ie, temporal segments). This deviation from the conventional assumption of consistent dataset distributions between training and testing data can lead to suboptimal prediction results. Secondly, even within the training set, these contexts essentially function as confounders \cite{pearl2009causality} --- factors that simultaneously influence the historical and future data, as demonstrated in Figure \ref{fig:intro-causal}. Such confounders lead trained models to capture spurious correlations, causing them to struggle with generalizing to data from new distributions. 

To present the impact of CDS in practice, we trained an Autoformer \cite{Autoformer} on \textit{Illness} and assessed its ability to fit sub-data in different periodic phases. Residuals, the difference between a model's prediction and the ground truth that can reflect the goodness of fitting, were analyzed for the 47th and 32nd periodic phases, and the entire dataset. The residual results are visualized in Figure \ref{fig:intro-residuals}. Notably, it can be observed that the model provides an \textit{unbiased} estimation for the entire training dataset (\ie, the mean of residuals is zero). However, the estimation within two specific contexts is \textit{biased} since both their means of residuals deviate from zero. This observation underscores the model's limitations in fitting sub-data within a context, as it is susceptible to data from other contexts and learns spurious correlations. It motivates us to calibrate the model to achieve more accurate estimation within each context.

\begin{figure}[t]
    \centering
    \includegraphics[width=0.45\textwidth]{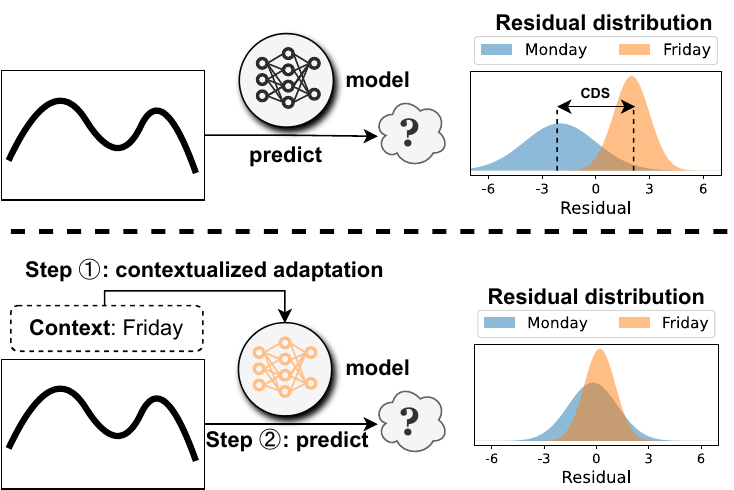}
    \caption{Illustrations of the traditional framework (top) and the proposed framework (bottom). By calibrating the model via contextualized adaptation before making each prediction, the context-driven distribution shift (CDS) can be alleviated.}
    \label{fig:intro-method}
\end{figure}

\stitle{Present work.} In this paper, we introduce a general calibration approach to detect and adapt to CDS with a trained model. Specifically, we first propose a metric to measure the severity of model's susceptibility to CDS within the training data, namely \textbf{\underline{Re}sidual-based \underline{con}text-driven \underline{di}stribu\underline{tion} shift detect\underline{or}} or \method{Reconditionor}, by measuring the mutual information between residuals and observed contexts to quantify the impact of contexts on model's prediction.

A higher value from \method{Reconditionor} signifies a stronger CDS. Under this circumstance, we further propose a simple yet effective adapter framework for further calibration. Given the inherent variability in contexts across data samples, a one-size-fits-all adaptation of the model is inherently unfeasible. Hence, we posit the need to fine-tune the model at the individual \textit{sample-level}. Figure \ref{fig:intro-method} illustrates the comparison between the traditional method and our proposed sample-level adaptation framework. Notably, for each test sample, adapting the model solely based on that single instance is intractable. As an alternative, we initiate a data augmentation process by curating a dataset comprising preceding samples characterized by akin contexts.
Since the chosen samples can introduce significant variance during the adaptation process, we restrict the fine-tuning to the model's prediction layer with a limited number of steps. Our theoretical findings substantiate that this approach can attain an optimal bias-variance trade-off.
We refer to this framework as \textbf{\underline{S}ample-level c\underline{O}ntextua\underline{LI}zed a\underline{D}apter}, or \method{SOLID}.

Extensive experiments indicate that our proposed calibration approach consistently enhances the performance of 7 forecasting models across 8 real-world datasets. Notably, our \method{Reconditionor} reliably identifies cases requiring CDS adaptation with a high accuracy of 89.3\%. Furthermore, our proposed \method{SOLID} yields an average improvement ranging from 8.7\% to 15.1\% when addressing significant CDS situations as detected by \method{Reconditionor}. Even in cases with less pronounced CDS, \method{SOLID} still achieves an average improvement ranging from 0.3\% to 6.3\%. From an efficiency perspective, our method introduces a 20\% additional time overhead, a gap smaller than the effects of dataset variability. Crucially, \method{Reconditionor}'s metric closely aligns with \method{SOLID}'s performance gains. These findings provide robust validation of the effectiveness of our calibration approach.

The main contributions of this work are summarized as follows: 
\begin{itemize}[leftmargin=*]
    \item We propose the concept of context-driven distribution shift (CDS) by studying the driving factors of distribution shifts and investigating two observed contexts (\textit{temporal segments} and \textit{periodic phases}), as well as unobserved contexts.
    \item We propose an end-to-end calibration approach, including \method{Reconditionor}, a detector to measure the severity of the model's susceptibility to CDS, and \method{SOLID}, an adapter to calibrate models for enhancing performance under severe CDS.
    \item Extensive experiments over various datasets demonstrate that \method{Reconditionor} detects CDS of forecasting models on the training dataset accurately, and \method{SOLID} significantly enhances current models without substantially compromising time efficiency.
\end{itemize}

%% file: sec-related_works.tex
\section{Related Work}

\subsection{Time series forecasting models}

With the successful rise of deep learning approaches, many recent work utilizes deep learning to better explore the non-linearity and multiple patterns of the time series and empirically show better performance. Some research introduces the Transformer to capture temporal dependencies with the attention mechanism. 
Specifically, Informer \cite{Informer} proposes ProbSparse self-attention. Autoformer \cite{Autoformer} introduces Auto-correlation attention based on seasonal-trend decomposition. FEDformer \cite{FEDformer} proposes Fourier frequency enhanced attention. ETSformer \cite{ETSformer} leverages exponential smoothing attention. Crossformer \cite{Crossformer} utilizes a two-stage attention to capture both cross-time and cross-dimension dependency. PatchTST \cite{PatchTST} proposes patching and channel independence techniques for better prediction. Different from the Transformer architecture, DLinear \cite{LTSF-Linear} leverages linear models with decomposition or normalization for TSF tasks. In this paper, our proposed pipeline can be easily applied to these TSF models, regardless of the model architecture.

\subsection{Distribution shift in time series}
\par Distribution shift in time series refers to the phenomenon that statistical properties and data distribution continuously vary over time.
To better detect such distribution shifts, various methods have been proposed. Stephan~\etal~\cite{Distribution_shift_detection_1} ~employs dimension reduction and proposes a two-sample hypothesis testing. Sean \etal ~\cite{Distribution_shift_detection_2} ~proposes an expected conditional distance test statistic and localizes the exact feature where the distribution shift appears. Lipton \etal ~\cite{Distribution_shift_detection_3} ~utilizes hypothesis testing based on their proposed Black Box Shift Estimation, thereby detecting the shift. However, these detection methods ignore the important context, and are mostly not appropriate for time series data. In contrast, our proposed residual-based detector has sufficient ability to detect the influence of underlying context on distribution shift.
\par Meanwhile, addressing distribution shifts in time series forecasting is also crucial.
One approach is to employ normalization techniques to stationarize the data. For example, DAIN \cite{DAIN} employs nonlinear networks to adaptively normalize time series. RevIN \cite{RevIn} proposes a reversible instance normalization to alleviate series shift. AdaRNN \cite{AdaRNN} introduces an Adaptive RNN to solve the problem. Dish-TS \cite{Dish-TS} proposes a Dual-Coefficient Net framework to separately learn the distribution of input and output space, thus capturing their divergence. Other approaches combine statistical methods with deep networks. Syml \cite{distribution_shift_2} applies Exponential smoothing on RNN, to concurrently fit seasonality and smoothing coefficients with RNN weights. SAF \cite{SAF} integrates a self-supervised learning stage on test samples to train the model before making the prediction. However, these methods are coupled to model architecture or require modification during or before the training process, which limits their application. In contrast, our proposed approach can efficiently adapt the given trained models solely during test time and at the sample level.

%% file: sec-preliminaries.tex
\section{Preliminaries}

For a multivariate time series with $M$ variables, let $\vx_{t} \in \mathbb{R}^M$ represent a sample at $t$-th timestep. Given a historical sequence: $\mX_{t-L:t} = [\vx_{t-L}, \cdots, \vx_{t-1}] \in \mathbb{R}^{L\times M}$, where $L$ is the look-back window size, the task is to predict future values with $T$ forecasting window size: $\hat{\mX}_{t:t+T} = [\vx_{t}, \cdots, \vx_{t+T-1}] \in \mathbb{R}^{T\times M}$. The training objective of a model $f$ is to find the best mapping from input to output sequence, \ie ~ $\hat{\mX}_{t:t+T} = f(\mX_{t-L:t})$. In this work, we assume $f$ is a deep neural network composed of two parts: a \textit{feature extractor} $g_\Phi: \mathbb{R}^{L\times M} \rightarrow \mathbb{R}^d$ mapping the historical values to a $d$-dimensional latent representation, and a linear top $h_\theta$ named as \textit{prediction layer} mapping the representation to predicted future values. The specific definitions of $g_\Phi$ and $h_\theta$ for different model architectures will be introduced in Appendix \ref{specific_structure_of_different_models}.

%% file: sec-methodology.tex
\section{context-driven distribution shift} \label{sec:context-limitation}

In this section, we introduce the limitation of traditional TSF models. As mentioned before, the data generation in time series is typically influenced by temporal external factors (\ie, context $c$), such as \textit{temporal segments} and \textit{periodic phases}. Let $X$, $Y$ and $C$ denote the variables of historical data $\mX_{t-L:t}$, future data $\mX_{t:t+T}$ and context $c_t$ at time step $t$, respectively. The generation of $Y$ is dependent on $X$ and $C$, characterized as $P(Y\mid X, C)$.

Due to the ignorance of contexts, the model $f$ trained on the dataset learns a marginal distribution $P(Y\mid X)=\sum_c P(Y\mid X, C=c)P(C=c)$. This 
introduces a confounding bias \cite{pearl2009causality} since context $C$ usually influences both the historical data $X$ and the future data $Y$, which is evident in Appendix \ref{sec:appendix_context_visual}. To illustrate this concept, consider a simplified example in the domain of recommendation. In the winter season (context $C$), users who purchase hot cocoa (historical data $X$) also tend to buy coats (future data $Y$). A model may "memorize" the correlation between $X$ and $Y$ (a spurious correlation), and mistakenly recommend coats to a user who purchases hot cocoa in summer. This confounding bias leads to suboptimal product recommendations. In our subsequent theoretical analysis (\sref{sec:theoretical_analysis}), we detail that such context consistently adds a bias term to the model. The following defines this phenomenon formally.
\begin{definition}[Context-driven Distribution Shift]
    \label{def:CDS}
    If there exists a variable $C$ (related to time $t$) that influences the data generation process from historical $X$ to future $Y$, \ie,
    \begin{align*}
        P(Y\mid X)\neq P(Y\mid X, C),
    \end{align*}
    then this time series data is said to exhibit a \textbf{context-driven distribution shift}, or CDS. The variable $C$ is termed the \textbf{context}.
\end{definition}

\begin{figure}[t]
    \centering
    \includegraphics[width=0.48\textwidth]{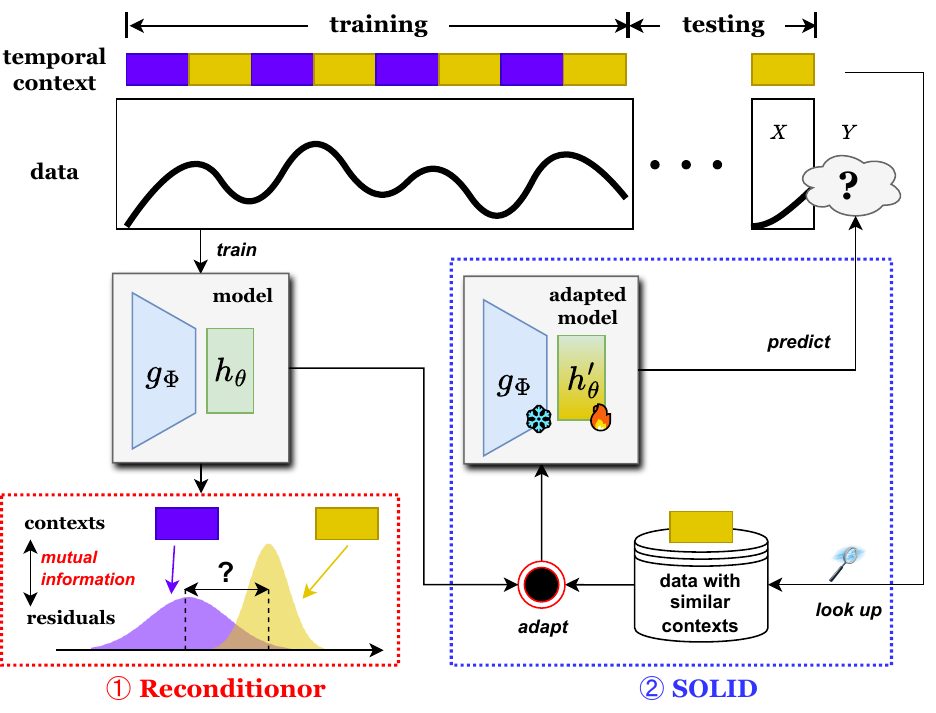}
    \caption{Pipeline of our calibration framework to detect and adapt to context-driven distribution shift (CDS). We leverage \textcircled{\raisebox{-0.9pt}{1}} \textbf{residual-based context-driven distribution shift detector} (\method{Reconditionor}) to assess how susceptible a trained model is to CDS. If we detect a significant susceptibility, we employ \textcircled{\raisebox{-0.9pt}{2}} \textbf{sample-level contextualized adapter} (\method{SOLID}) to adapt the model for each test sample using preceding data that share similar contexts.}
    \label{fig:pipeline}
\end{figure}

\section{Calibration framework for CDS}

In this section, we introduce a general calibration methodology for detecting (in \sref{sec:context-indicator}) and adapting (in \sref{sec:solid}) to CDS in conjunction with the model. The pipeline of this calibration framework is illustrated in Figure \ref{fig:pipeline}.

\subsection{Residual-based CDS detector}\label{sec:context-indicator}

Our primary focus lies in assessing the model's susceptibility to CDS. Our evaluation predominantly centers around observed contexts, as the analysis of unobserved contexts is computationally infeasible. Fortunately, for our empirical investigation, the utilization of observed contexts proves to be sufficient and effective.

As visually demonstrated in Figure \ref{fig:intro-residuals}, the presence of contexts introduces a bias to the model estimation, causing variations in residual distributions across different contexts. Based on it, we propose a novel detector, namely \textbf{\underline{Re}sidual-based \underline{con}text-driven \underline{di}stribu\underline{tion} shift detect\underline{or}} (or \method{Reconditionor}), by measuring the \textit{mutual information} (MI) between prediction residuals and their corresponding contexts. The MI quantifies the extent of information acquired regarding the residuals when observing the context, which serves as a metric for evaluating the influence of contexts on the model. MI can be computed by:
\begin{align}
    \label{eq:indicator-residual}
    \delta = \text{MI}(\Delta Y; C)
    = \mathbb E_{C} \left[
        D_{\text{KL}} \left(P(\Delta Y \mid C)\ \| \ P(\Delta Y)
    \right)\right],
\end{align}
where $\Delta Y=f(X)-Y$ is the residuals of model $f$, and $D_{\text{KL}}(P \| Q)$ is Kullback–Leibler (KL) divergence between distributions $P$ and $Q$.

We reuse Figure \ref{fig:intro-residuals} to illustrate the concept behind \method{Reconditionor} when detecting the distribution shift based on the context of periodic phases. The marginal residual distribution $P(\Delta Y)$ typically exhibits a mean close to zero after training. However, the residual distributions conditioned on different contexts (\eg, the 47th phase and the 32nd phase) $P(\Delta Y\mid C)$ clearly show non-zero mean values. This increases the KL divergence between the two distributions, consequently elevating $\delta$ and indicating a strong CDS. Additionally, a non-zero mean in the conditional residual distribution suggests the model $f$ fails to adequately fit the data within each context, signaling the need for further adaptation. In summary, a high value of $\delta$ for a model $f$ implies the necessity for adapting $f$, which is also empirically verified in \sref{sec:relation_reconditionor_solid}.

In practice, we assume that the residuals follow Gaussian distributions. This assumption is based on the utilization of MSE loss, which implicitly presupposes that the residuals adhere to additive Gaussian noise. This characteristic is also evident in Figure \ref{fig:intro-residuals}. The adoption of this assumption expedites the calculation of KL divergence because Gaussian distributions offer a straightforward analytical solution for it.

We illustrate the full algorithm for \method{Reconditinor} in Algorithm \ref{alg:Reconditionor}. In lines 1-2, we initialize the sets of residuals for both the marginal distribution $P(\Delta Y)$ and the conditional distributions $P(\Delta Y\mid C)$ for each $C\in[K]$. In lines 3-7, we update these sets with the residuals computed by $f$. We compute the mean and standard deviation values for $P(\Delta Y)$ in line 8 and perform a similar computation for $P(\Delta Y\mid C)$ in line 11. Finally, in line 12, we compute and average the KL divergences between $P(\Delta Y)$ and $P(\Delta Y\mid C)$ to obtain the detector score $\delta$.

\begin{algorithm}[t]
\caption{Algorithm for \method{Reconditionor}}
\label{alg:Reconditionor}
\KwIn{Model $f$, training data with $K$ contexts $\mathcal D^{\text{train}}=\{(\mX_{t-L:t}, \mX_{t:t+T}, c_{t}): t<t_{\text{train}}, c_t\in[K]\}$.}
\KwOut{$\delta\in[0, 1]$ indicating $f$'s susceptibility to CDS.}
$R \gets \varnothing$\;
$R_1, \cdots, R_{K} \gets \varnothing, \cdots,  \varnothing$\;
\For{$L \leq t < t_{\text{train}}$}{
    $r\gets f(\mX_{t-L:t}) - \mX_{t:t+T}$\;
    $R\gets R \cup r$\;
    $R_{c_t}\gets R_{c_t} \cup r$\;
}
$\mu, \sigma\gets \text{Mean}(R), \text{Standard-Deviation}(R)$\;
$\delta\gets 0$\;
\For{$c \in [K]$}{
    $\mu_c, \sigma_c\gets \text{Mean}(R_c), \text{Standard-Deviation}(R_c)$\;
    $\delta\gets \delta + \frac{|R_c|}{|R|} \text{KL}(\mathcal N(\mu_c, \sigma_c^2)\ \|\ \mathcal N(\mu, \sigma^2))$\;
}
\Return{$\delta$}\;
\end{algorithm}

\subsection{Sample-level contextualized adapter}\label{sec:solid}

A higher metric $\delta$ from \method{Reconditionor} signifies a stronger impact of CDS on a model. Given the nature of CDS, our primary concept is to adjust the model to align with the conditional distribution $P(Y | X, C)$ instead of the marginal distribution $P(Y | X)$. However, noticing that the context $C$ is consistently changing at each time step, a one-size-fits-all adaptation of the model is inherently unfeasible. Therefore, we propose to carry out adaptations at the individual \textit{sample-level}.

For each test sample $X_{t-L:t}$, it is not viable to adapt the model solely relying on the input $X_{t-L:t}$. Therefore, we commence by implementing data augmentation through the creation of a dataset derived from this specific sample, formulated as:
\begin{align}
\label{eq:contextualized_dataset}
    \mathcal D_{\text{ctx}} = \textsc{Select}(\{(\mX_{t'-L:t'}, \mX_{t':t'+T}): t' + T\leq t\}),
\end{align}
where $\textsc{Select}$ operation involves the selection of preceding samples that share a similar context with the provided sample $\mX_{t-L:t}$, and we will provide further elaboration on this operation in \sref{sec:designed_index}. We denote the resulting dataset as the \textit{contextualized dataset} ($\mathcal D_{\text{ctx}}$).  Specifically, before making prediction to the test sample $X_{t-L:t}$, we employ $\mathcal D_{\text{ctx}}$ to adapt the model to alleviate the influence of CDS. We refer to this step as \textbf{\underline{S}ample-level c\underline{O}ntextua\underline{LI}zed a\underline{D}apter}, or \method{SOLID}. It's worth noting that since adaptation takes place during the testing phase, we propose to modify the prediction layer $h_\theta$ while keeping the feature extractor $g_\Phi$ unchanged for efficiency. In our empirical analysis (\sref{sec:exp_tuning_strategies}), we observed that fine-tuning solely the prediction layer not only improves test phase efficiency but also consistently delivers better performance by mitigating the risk of overfitting.

Additionally, as we outline in the subsequent theoretical analysis (\sref{sec:theoretical_analysis}), the fine-tuning process is fundamentally a \textit{bias-variance trade-off}: fine-tuning reduces the bias caused by CDS, but introduces an additional variance from the noise of contextualized dataset due to the data scarcity, potentially impacting model performance negatively. Thus, optimally tuning the fine-tuning steps to balance the trade-off is essential.

\subsection{Contextualized dataset selection} \label{sec:designed_index}

As we mentioned previously, the core of adaptation involves creating the contextualized dataset $\mathcal D_{\text{ctx}}$ for the sample at $t$. In this section, we introduce the $\textsc{Select}$ operation in \Eqref{eq:contextualized_dataset}.  Note that due to the unavailability of the true context governing the data generation process, it is not feasible to select samples with precisely the same context. To address this issue, we design a comprehensive strategy based on the observable contexts (temporal segments and periodic phases), and employ sample similarity as a proxy for unobserved contexts.

\subsubsection{Temporal segments}
\label{sec:index_1}

The data generation process typically evolves over time \cite{distribution-shift}. Consequently, we claim that the \textbf{temporal segment} is a critical context. Therefore, we focus on samples that are closely aligned with the test samples in the temporal dimension, formally,
\begin{align*}
    \{(\mX_{t'-L:t'}, \mX_{t':t'+T}): t - \lambda_T \leq t'\leq t - T\},
\end{align*}
where $\lambda_T$ controls the time range for selection. When samples are too distant from $t$, we conclude that they are in distinct temporal segments and consequently exclude them from selection.

\subsubsection{Periodic phases}
\label{sec:index_2}

Furthermore, it's worth mentioning that time series data often exhibit periodic characteristics. The data generation process can vary across different phases. Therefore, we claim that the \textbf{periodic phase} constitutes another critical context.

To find the samples with similar phases, we need to detect the underlying periods. Specifically, we follow ETSformer \cite{ETSformer} and TimesNet \cite{TimesNet} to employ the widely-used Fast Fourier Transform (FFT) on the training dataset $\mX\in\mathbb R^{t_{\text{train}}\times M}$ with length-$t_{\text{train}}$ and $M$ variables, formulated as:
\begin{align}
    \label{eq:period-1}
    T^* = \Bigl\lfloor\nicefrac{t_{\text{train}}}{\left\{\mathop{\arg\max}_{
        k\in\left\{2,...,\left[ \nicefrac{t_{\text{train}}}{2} \right]\right\}
    } \sum_{i=1}^M \text{Ampl}\left(
        \text{FFT}\left(
            \mX^{i}
        \right)
    \right)_k\right\}} \Bigr\rfloor.
\end{align}

Here, $\mX^{i}$ is the sequence of the $i$-th variables in $\mX$, $\text{FFT}(\cdot)$ and $\text{Ampl}(\cdot)$ denote the Fast Fourier Transform (FFT) and amplitude respectively. To determine the most dominant frequency, we sum the amplitudes across all $M$ channels and select the highest value, which is converted to the periodic length $T^*$.

Next, we employ the periodic length to select samples with closely aligned phases. Particularly, for the given test sample at time step $t'$, we select samples that display minimal difference in the phases, formulated as
\begin{align*}
    \left\{(\mX_{t'-L:t'}, \mX_{t':t'+T}): 
        \left|
            \frac{t \text{ mod } T^* - t' \text{ mod } T^*}{T^*}
        \right| < \lambda_P
    \right\},
\end{align*}
where $t \text{ mod } T^*$ and $t' \text{ mod } T^*$ are the phases of the test sample and preceding samples, respectively. $\lambda_P$ is a hyperparameter controlling the threshold for the acceptable phase difference. If the difference exceeds a certain threshold, the preceding samples will not be considered to share the same context as the test sample.

\begin{algorithm}[t]
\caption{Algorithm for \method{SOLID}}
\label{alg:solid}
\KwIn{Model $f=(g_\Phi, h_\theta)$, test sample $\mX_{t-L:t}$, preceding data $\{(\mX_{t'-L:t'}, \mX_{t':t'+T}): t'+T\leq t\}$, similarity metric $S(\cdot, \cdot)$, periodic length $T^*$ computed by \Eqref{eq:period-1}, hyperparameters $\lambda_T$, $\lambda_P$, $\lambda_N$ and $lr$.}
\KwOut{Prediction for the test sample: $\hat \mX_{t:t+T}$}
$\mathcal T \gets \varnothing$\;
\For{$t-\lambda_T \leq t' \leq t-T$}{
    $\Delta_P \gets \left|
        \frac{t \text{ mod } T^* - t' \text{ mod } T^*}{T^*}
    \right|$\;
    \If{$\Delta_P < \lambda_P$}{
        $\mathcal T \gets \mathcal T \cup \{t'\}$\;
    }
}
$\mathcal T_{\text{ctx}} \gets \mathop{\text{argTop-}\lambda_N}\limits_{t'\in \mathcal T}(S(\mX_{t'-L:t'}, \mX_{t-L:t}))$\;
$\mathcal D_{\text{ctx}} \gets \{
    (\mX_{t'-L:t'}, \mX_{t':t'+T})\mid t\in \mathcal T_{\text{ctx}}
\}$\;
$h_\theta' \gets$ fine-tune $h_\theta$ using $\mathcal D_{\text{ctx}}$ with a learning rate $lr$\;
$\hat \mX_{t:t+T} \gets h_\theta'(g_\Phi(\mX_{t-L:t}))$\;
\Return{$\hat \mX_{t:t+T}$}\;
\end{algorithm}

\subsubsection{Address unobserved contexts through sample similarity}
\label{sec:index_3}

Even though the strategies introduced in \sref{sec:index_1} and \sref{sec:index_2} efficiently identify potential samples with similar contexts, we cannot guarantee a consistent mapping relationship $X\mapsto Y$ for these samples due to the existence of \textit{unobserved contexts}. To further enhance the quality of selection and address this issue, it's essential to recognize that context typically influences input data through a causal effect $C\mapsto X$, which suggests a correlation between contexts and inputs. 

Inspired by this insight, we assume that when samples have similar inputs $X$, they are more likely to share a similar context $C$. Consequently, we incorporate \textbf{sample similarity} as a proxy of unobserved contexts. The calculation of the similarity can be any measurement of interest, and we employ the Euclidean distance as the chosen metric. Specifically, we select the top-$\lambda_N$ similar samples, where $\lambda_N$ serves as a hyperparameter governing the number of samples to be chosen.

\subsubsection{Full algorithm for \method{SOLID} }

Finally, we combine the above strategies for the $\textsc{Select}$ operation, by first filtering the temporal segments and periodic phases (\sref{sec:index_1} and \sref{sec:index_2}) and then selecting top-$\lambda_N$ samples based on sample similarity (\sref{sec:index_3}).

We illustrate the full algorithm for \method{SOLID} in Algorithm \ref{alg:solid}. We first compute the periodic length on the training dataset (\Eqref{eq:period-1}). During the testing stage, for a given test sample, we perform the following steps: In lines 1-7, we filter the time steps based on the observed contexts, temporal segments (\sref{sec:index_1}), and periodic phases (\sref{sec:index_2}). In lines 8-9, we further select samples based on similarity (\sref{sec:index_3}). In lines 10-11, we fine-tune the prediction layer (\sref{sec:solid}) and use it for making predictions.

\section{Theoretical analysis}\label{sec:theoretical_analysis}

In this section, we provide a formal theoretical analysis to estimate the generalization error, before and after considering context, to illustrate the influence from CDS (\sref{sec:context-limitation}), as well as the bias-variance trade-off during the fine-tuning process (\sref{sec:solid}). Given our fine-tuning targets the prediction layer with the feature extractor remaining frozen, our theoretical analysis centers on the latent representation space $g_\Phi(X)$ rather than the raw data $X$. To start with, we first make the following assumption for the generation process of $Y$:

\begin{assumption}[Contextualized generation process]
    \label{asm:data}
    Assume that the input latent representations $g_\Phi(X)$ on the training set can be divided into $K$ context groups based on $K$ different contexts: $(X_1, \cdots, X_K)$, where $X_i \in \mathbb R^{n_i\times d}$ and $n_i$ is the number of data points in the $i$-th group. For each $i$, there exists a parameter vector $\theta_i\in\mathbb R^{d}$ such that the output $Y_i$ follows:
    \begin{align*}
        Y_i = X_i \theta_i + \epsilon_i,
    \end{align*}
    where $\epsilon_i$ is an independent random noise, which satisfies $\mathbb E[\epsilon_i] = 0$, $\mathbb{VAR}[\epsilon_i] = \sigma^2$. Here we assume that $Y_i$ are scalars for simplicity, although it can be readily extended to a multi-dimensional scenario.
\end{assumption}

This assumption extends the widely used \textit{fixed design setting} \cite[Chapter 3.5]{bach2023learning} to multi-context scenarios, which posits that the data generation parameters, $\theta_i$, differ across various contexts $i$. The prediction layer $h_\theta$ directly trained without considering the contexts can be seen as a \textit{global linear regressor} (GLR), as follows:
\begin{definition}[Global linear regressor]
    \label{def:global-linear-regressor}
    A global linear regressor (GLR) $h_{\hat\theta}$ parameterized by $\hat \theta$ is given by:
    \begin{align*}
        \hat \theta = \arg\min_{\theta}\sum_{i=1}^K{||Y_i - X_i\theta||_2^2}.
    \end{align*}
\end{definition}

Define $\mathcal{R}(\alpha_1, \cdots, \alpha_K) = \mathbb{E} \left[ \sum_{i = 1}^K ||Y_i - X_i \alpha_i||_2^2 \right]$ as the expected risk when using a parameter $\alpha_i$ to predict $X_i$ ($i\in[K]$), and let $\mathcal{R}^*$ denote the minimum value of this risk. The following theorem computes the expected risk when we use the globally shared parameter for prediction, with its proof delegated to Appendix \ref{sec:appendix_proof}.
\begin{theorem}[Expected risk for GLR]
    \label{thm:error-global}
    For $\hat\theta$ in Definition \ref{def:global-linear-regressor}:
    \begin{align*}
        \mathcal{R}(\underbrace{\hat \theta, \cdots, \hat \theta}_{K}) - \mathcal{R}^* = \underbrace{\sum_{i=1}^K\left\|
            \overline\theta-  \theta_i
        \right\|^2_{\psi_i}}_{\text{bias part}} + 
        \underbrace{\sigma^2 d}_{\text{variance part}},
    \end{align*}
    where $\psi_i = X_i^\top X_i$, and $\overline \theta = (\sum_{i=1}^K{\psi_i})^{-1} (\sum_{i=1}^K{\psi_i \theta_i})$. The quantity $||\cdot||_{\psi_i}$ is the Mahalanobis distance norm, defined as $ \|\theta\|^2_{\psi_i} = \theta^\top \psi_i \theta$.
\end{theorem}

The bias part in Theorem \ref{thm:error-global} indicates that GLR is unbiased only when the data generation parameters $\theta_i$ are identical across all groups. However, if they differ due to the influence of contexts, the regressor is biased regardless of the amount of data, \ie, CDS.

In the next, we explore a straightforward approach to address CDS: discarding the existing biased regressor and training a new individual regressor $\hat \theta_i$ for each context group $X_i$ ($i\in [K]$) to eliminate the bias. We refer to this ensemble of regressors as \textit{contextualized linear regressors} (CLR), as follows:
\begin{definition}[Contextualized linear regressor]
    \label{def:context-linear-regressor}
    A set of contextualized linear regressors (CLR) $h_{\hat\theta_i}$ parameterized by $\hat\theta_i$ ($i\in[K]$) are given by:
    \begin{align*}
        \hat \theta_i = \arg\min_{\theta}||Y_i - X_i\theta||_2^2, \quad\forall i \in \{1,2,...,K\}.
    \end{align*}
\end{definition}

Using the same risk notations, the following theorem computes the expected risk for CLR (the proof is detailed in Appendix \ref{sec:appendix_proof}).
\begin{theorem}[Expected risk for CLR]\label{thm:error-context}
    For $\hat\theta_i$ in Definition \ref{def:context-linear-regressor}:
    \begin{align*}
        \mathcal{R}(\hat \theta_1, \cdots, \hat \theta_K) - \mathcal{R}^* = \underbrace{0}_{\text{bias part}} + \underbrace{K \sigma^2 d}_{\text{variance part}}.
    \end{align*}
\end{theorem}

Comparing Theorem \ref{thm:error-global} and Theorem \ref{thm:error-context}, we observe that CLR is always \emph{unbiased} since it addresses the CDS. However, it suffers from a \emph{larger variance} compared to GLR. Specifically, the variance of CLR is $K$-times larger than that of GLR, indicating that more detailed contexts result in higher variance. This makes sense because as $K$ increases, the number of data available for training each CLR diminishes, consequently elevating the variance.

Based on the above findings, we argue that it's crucial to combine GLR and CLR to balance the bias and variance. This can be implemented through a standard pre-training / fine-tuning paradigm. In the pre-training stage, contexts are disregarded and a GLR $h_\theta$ is trained on the training dataset. This pre-training stage mirrors the conventional standard training process. In the fine-tuning stage, for a given new sample at $t$, we employ the dataset with the same context as this sample to fine-tune the learned GLR for limited steps. This stage mirrors the usage of the proposed \method{SOLID}, which brings GLR closer to CLR and reduces bias. In cases where the influence of CDS is substantial (\ie, the bias part in Theorem \ref{thm:error-global} is large), it is advisable to increase the learning rate and the number of fine-tuning steps to mitigate bias. In the converse case, the fine-tuning process should be more limited to reduce variance, as CLR has greater variance (Theorem \ref{thm:error-context}). In practice, it is recommended to tune the learning hyperparameters of \method{SOLID} to achieve an optimal trade-off between bias and variance.

%% file: sec-experiments.tex
\section{Experiments}

\input{figure/main_table}

In this section, we describe the experimental settings and provide extensive results and analysis. \footnote{Code is available at \url{https://github.com/HALF111/calibration_CDS}.}

\subsection{Experiment settings}

\textbf{Datasets.} We conduct the experiments on 8 popular datasets for TSF: Electricity, Traffic, Illness, Weather \cite{LSTNet}, and 4 ETT datasets (ETTh1, ETTh2, ETTm1, ETTm2) \cite{Informer}. We follow the standard preprocessing protocol \cite{Informer, Autoformer} and partition the datasets into train/validation/test sets by the ratio of 6:2:2 for ETT and 7:1:2 for the other datasets. \ifkdd{}\else{Appendix \ref{dataset_detailed} contains more dataset details.}\fi

\par \textbf{Baseline models.} As aforementioned, our proposed approach is a general framework that can calibrate many deep forecasting models. To verify the effectiveness, we utilize several forecasting models for the detection and adaptation, including Informer \cite{Informer}, Autoformer \cite{Autoformer}, FEDformer \cite{FEDformer}, ETSformer \cite{ETSformer}, Crossformer \cite{Crossformer}, DLinear \cite{LTSF-Linear} and PatchTST \cite{PatchTST}. \ifkdd{}\else{Appendix \ref{baseline_models} contains more details about the baselines.}\fi

\par \textbf{Experimental details.} For a fair comparison, we set the prediction length $T$ of \{24,36,48,60\} for the Illness dataset, and \{96,192,336,720\} for the others, which aligns with the common setting for TSF tasks \cite{Autoformer, FEDformer}. Additionally, for other hyper-parameters, we follow the primary settings proposed in each respective paper \cite{Crossformer, ETSformer, FEDformer, Autoformer, Informer}. For our \method{SOLID}, we employ gradient descent to fine-tune the prediction layer on $\mathcal D_{\text{ctx}}$ for one epoch. Appendix \ref{hyper-parameters} contains more details about the hyper-parameters.
\par \textbf{Evaluation metrics.} Consistent with previous research \cite{Informer, Autoformer, FEDformer}, we compare the performance via two widely-used metrics: Mean Squared Error (MSE) and Mean Absolute Error (MAE). Smaller MSE and MAE indicate better forecasting performance. We also compute the proposed \method{Reconditionor} as a metric for detecting the extent of models' susceptibility to CDS based on the two observed contexts, periodic phase ($\delta_P$) and temporal segment ($\delta_T$). For $\delta_P$, the number of contexts equals to the periodic length, detailed in \Eqref{eq:period-1}. For $\delta_T$, we partition the training set into five equal segments, with each one representing a distinct temporal segment. Larger $\delta_P$ and $\delta_T$ indicate a stronger CDS for the given model and dataset.
\subsection{Main results analysis}\label{sec:main_result}

Table \ref{tbl:Main_Results_Part1} shows the main results of our proposed \method{SOLID} alongside the scores of the proposed \method{Reconditionor}. We conduct a thorough analysis of the data from both perspectives.

\subsubsection{Effectiveness of \method{SOLID}} 
\begin{enumerate}[leftmargin = 1em, itemindent = 0em]
    \item[$\bullet$] \method{SOLID} enhances the performance of all models across various datasets by effectively addressing the CDS issue. Particularly, \method{Informer}, which is relatively less powerful, experiences a significant enhancement of 10\%-60\% in MSE. One explanation is that CDS has a more detrimental effect on weaker models, as they struggle to learn diverse patterns under different contexts.
    \item[$\bullet$] For datasets and models showing significant CDS under the context of periodic phases (\ie, $\log_{10}\delta_P\geq -3.2$), \method{SOLID} achieves a considerable improvement. Specifically, it yields an average improvement of 15.1\% for Illness, 9.9\% for Traffic, and 8.7\% for Electricity in terms of MSE on these cases. This highlights the effectiveness of \method{SOLID} in addressing severe CDS issues.
    \item[$\bullet$] On the cases of weak CDS under the context of periodic phases (\ie, $\log_{10}\delta_P$ below the $-3.2$ threshold), \method{SOLID} still achieves a consistent improvement ranging from 0.3\% to 6.7\% across different datasets. This observation underscores the widespread presence of CDS in time series.
\end{enumerate}

\subsubsection{Effectiveness of \method{Reconditionor}}

\begin{enumerate}[leftmargin = 1em, itemindent = 0em]
    \item[$\bullet$] Our \method{Reconditionor} effectively assesses the magnitude of CDS and aligns with the enhancement achieved by \method{SOLID}. When employing the $-3.2$ threshold for $\log_{10} \delta_P$ to ascertain whether MAE improvement surpasses 1\%, the classification accuracy reaches 89.3\% (50 out of 56). This demonstrates that this detector is universally applicable and unaffected by different datasets and models.
    \item[$\bullet$] One exception is that \method{Informer} consistently achieves an improvement above  1\%, irrespective of \method{Reconditionor} metrics. As aforementioned, one explanation is that Informer is relatively weaker, making it more amenable to improvement by \method{SOLID}. When the results of \method{Informer} are excluded, the classification accuracy for \method{Reconditionor} rises to 93.8\% (45 out of 48).
    \item[$\bullet$] In addition, $\delta_P$ explains the performance improvement better than $\delta_T$. We will give a possible conjecture in the following subsection to elucidate this observation.
\end{enumerate}

\subsubsection{Correlation of \method{Reconditionor} and \method{SOLID}}\label{sec:relation_reconditionor_solid}

\begin{figure}[t]
    \centering
    \includegraphics[width=0.47\textwidth,trim={10 10 10 10},clip]{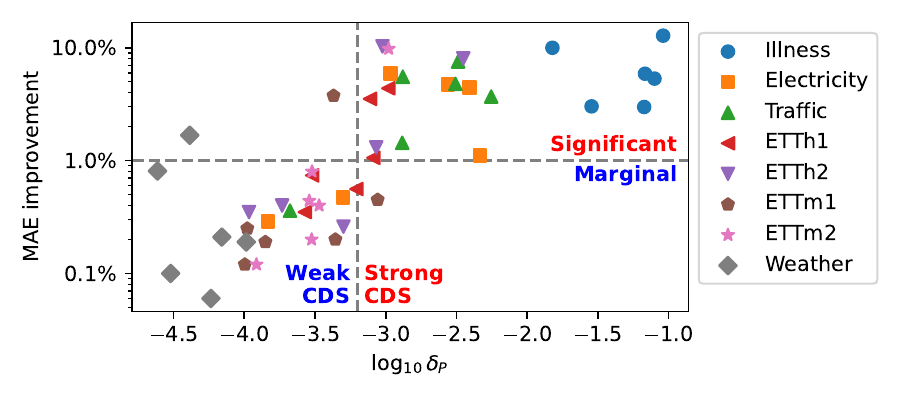}
    \caption{The relationship of \method{Reconditionor} metric $\bm{\log_{10} \delta_P}$ (X-axis) and MAE improvements achieved by \method{SOLID} (Y-axis) for 8 datasets and 6 models. }
    \label{fig:indicator-periodic}
\end{figure}

To further investigate the correlation between \method{Reconditionor} and the improvements achieved by \method{SOLID}, we plot $\delta_P$ and MAE improvements in Figure \ref{fig:indicator-periodic} based on the results of Table \ref{tbl:Main_Results_Part1} and Table \ref{tbl:Main_Results_Part2} in Appendix \ref{sec:appendix_full_result}. To depict the trend more effectively, we have excluded the data of \method{Informer} as previously explained.
Notably, we have observed a pronounced and consistent upward trend between $\delta_P$ and MAE improvement in Figure \ref{fig:indicator-periodic}. This trend is highly evident, with a Spearman correlation coefficient of 0.7998. This finding indicates that $\delta_P$ can serve as a valuable metric for estimating and explaining performance improvements. We also elaborate on the relationship between $\delta_T$ and MAE improvement in detail in Appendix \ref{sec:app_corr_t}.

\begin{figure*}[t]
    \centering
    \subfigure[Ablation (illness)]{
        \label{fig:extra-ablation-illness}
        \includegraphics[width=0.158\textwidth,trim={10 10 10 10},clip]{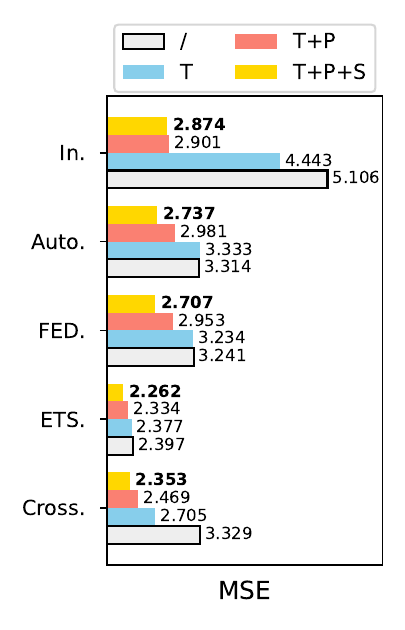}
    }%
    \subfigure[Ablation (traffic)]{
        \label{fig:extra-ablation-traffic}
        \includegraphics[width=0.158\textwidth,trim={10 10 10 10},clip]{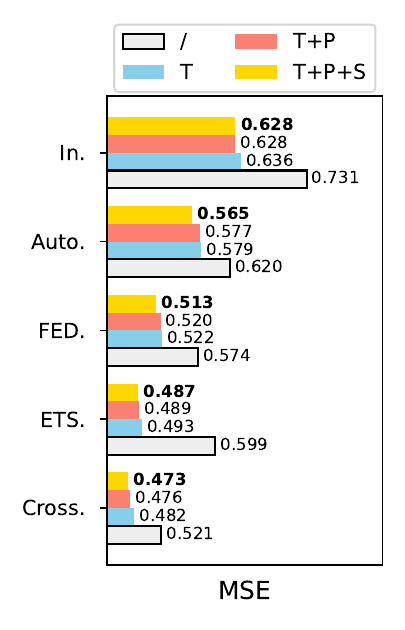}
    }%
    \subfigure[Tuning strategies]{
        \label{fig:extra-tuning}
        \includegraphics[width=0.158\textwidth,trim={10 10 10 10},clip]{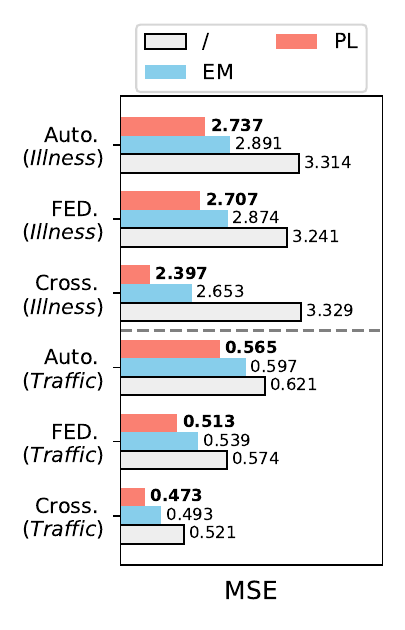}
    }%
    \subfigure[RevIN]{
        \label{fig:extra-revin}
        \includegraphics[width=0.158\textwidth,trim={10 10 10 10},clip]{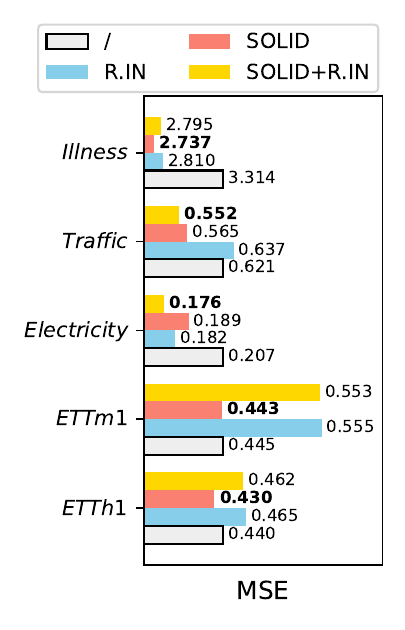}
    }%
    \subfigure[Dish-TS]{
        \label{fig:extra-dishts}
        \includegraphics[width=0.158\textwidth,trim={10 10 10 10},clip]{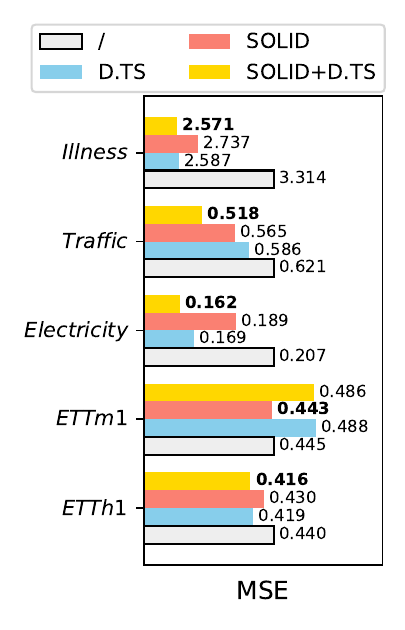}
    }%
    \subfigure[Speed]{
        \label{fig:extra-speed}
        \includegraphics[width=0.158\textwidth,trim={10 10 10 10},clip]{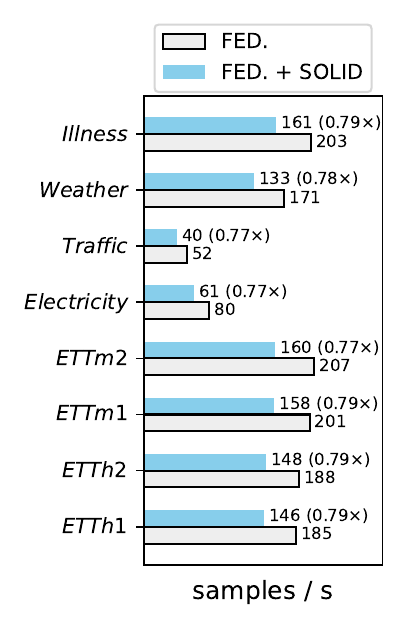}
    }%
    \label{fig:extra}
    \vspace{-0.3cm}
    \caption{Further analysis of \method{SOLID}. (\underline{a})(\underline{b}) Ablation studies for three contexts, temporal segment (T), periodic phase (P), and sample similarity (S). (\underline{c}) Studies on tuning strategies to explore adaptation on prediction layer (PL) only vs. entire model (EM). (\underline{d})(\underline{e}) Comparison studies against RevIN (R.IN) and Dish-TS (D.TS). (\underline{f}) Efficiency studies on the prediction speed.} \label{Comprehensive_Analysis}
\end{figure*}

\subsection{Further analysis for \method{SOLID}}\label{sec:further_analysis}

In this section, we conducted additional experiments to further analyze \method{SOLID}. \ifkdd{Figure \ref{Comprehensive_Analysis} presents results.}\else{Figure \ref{Comprehensive_Analysis} presents partial results, with the full results reported in Appendix \ref{sec:appendix_full_result}.}\fi

\subsubsection{Ablation studies}

We conducted ablation studies to investigate the relationship between performance and each component in the selection strategy, particularly each observable or unobservable context we propose in \sref{sec:designed_index}. The investigation involves a progression of strategy components. Specifically, we commence from original forecasts that ignore any contexts, denoted as \textbf{"/"}. We then create contextualized datasets with various selection variations. For a fair comparison, all datasets are constructed by selecting $\lambda_N$ samples from $\lambda_T$ nearest ones. We first consider \textbf{\underline{T}emporal segment} context by selecting the nearest $\lambda_N$ samples, denoted as \textbf{"T"}. Then we take \textbf{\underline{P}eriodic phase} into consideration, and select $\lambda_N$ nearest samples with phase differences less than $\lambda_P$, denoted as \textbf{"T+P"}. Finally, we add \textbf{sample \underline{S}imilarity} and select the top-$\lambda_N$ similar samples from nearest $\lambda_T$ samples with phase differences less than $\lambda_P$. This strategy mirrors \method{SOLID} and is denoted as \textbf{"T+P+S"}.
Based on the results in Figure \ref{fig:extra-ablation-illness} and \ref{fig:extra-ablation-traffic}, we showcase that each component leads to a consistent improvement in both MSE and MAE. These findings provide compelling evidence supporting the effectiveness of all components utilized in our selection strategy as well as the \method{SOLID} method.

\subsubsection{Comparative studies on adapting only prediction layer or entire model}\label{sec:exp_tuning_strategies}

All our previous experiments are performed solely by adapting the prediction layer while keeping the crucial parameters of bottom feature extractor layers unchanged. Nevertheless, it is imperative to conduct comparative experiments to validate this perspective. Specifically, we compare the results achieved by solely adapting the prediction layer against adapting the entire model. The experimental results are presented in Figure \ref{fig:extra-tuning}. It indicates that solely adapting the prediction layer consistently yields superior performance. One explanation is that it reduces the risk of over-fitting. Moreover, solely adapting the prediction layer significantly faster the speed of inference.

\subsubsection{Comparison studies against other approaches addressing distribution shifts}

In this section, we compared two widely used baselines, RevIN \cite{RevIn} and Dish-TS \cite{Dish-TS} which also address distribution shifts in time series forecasting. It is worth mentioning that they are strategies that work during training, which differs from our \method{SOLID} acting as a post-calibration method during the test process. Also, both RevIN and Dish-TS adapt the \textit{input data} to better suit the model, whereas SOLID adapts the \textit{model} to match the data distribution of the current context. Therefore, these methods are complementary and not exclusive.
We explored their joint use on Autoformer, shown in Figures \ref{fig:extra-revin} and \ref{fig:extra-dishts}. We can observe that \method{SOLID} consistently improves performance, regardless of whether models were trained with RevIN or Dish-TS. Moreover, in many cases, combining them achieves the best performance, suggesting the orthogonality of Context-driven DS (CDS) with other types of DS addressed by them.

\subsubsection{Efficiency analysis}

Furthermore, we analyze the practical efficiency of the proposed framework. For \method{Reconditionor}, it conducts a one-time analysis of residuals during the training process, thus not incurring any additional overhead during deployment. For \method{SOLID}, it incurs two steps of overhead for each prediction sample: (1) constructing the contextualized dataset, and (2) fine-tuning the model. Step (1) can be optimized using some engineering techniques, such as pre-caching intermediate embeddings and building efficient vector indexing. For step (2), since only the prediction layer is fine-tuned, the computational cost will not be too high. To further assess \method{SOLID}'s efficiency, we report the additional time overhead introduced by \method{SOLID} to FEDFormer across various datasets in Figure \ref{fig:extra-speed}. The results indicate that \method{SOLID}'s inference speed is about 80\% of traditional methods, a gap smaller than the variability introduced by different datasets, which suggests that our method does not significantly reduce inference speed.

\newcommand{\rotateablation}[1]{{#1}}

\subsubsection{Parameter sensitivity analysis and case study} 
We adjusted the hyperparameters within \method{SOLID} across multiple datasets and models to assess their impact on performance and conducted a case study to visualize the improvement. From the results, we discover that \method{SOLID} is insensitive to $\lambda_T$, which controls the time range of preceding data for selection, and $\lambda_P$, which governs the acceptable threshold for periodic phase difference. But concerning $\lambda_N$, which determines the number of similar samples to be selected for model adaptation, and $lr$, which regulates the extent of adaptation on the prediction layer for the models, they exist an optimal value and should be well selected. The detailed results are reported in Appendix \ref{sec:appendix-parameter-sensitivity} and Appendix \ref{sec:appendix-visulization}.

%% file: figure/main_table.tex
\newcommand{\rotate}[1]{\begin{tabular}{@{}c@{}}\multirow{6}{*}{{\rotatebox[origin=c]{90}{#1}}}\end{tabular}}

\begin{table*}[htbp]
  \centering  
  \caption{Performance comparison. "24 / 96": prediction length is 24 (Illness) or 96 (other datasets), applicable similarly to "36 / 192", \textit{etc}. "$\bm \uparrow$": average improvements achieved by \method{\textit{SOLID}} compared to the baseline. "$\bm \delta$": metrics given by \method{\textit{Reconditioner}} in two observed contexts: periodic phases ($\bm{\delta_P}$) and temporal segments ($\bm{\delta_T}$), in the form of $\bm{\log_{10} \delta_P\ \&\ \log_{10} \delta_T}$. \textcolor[rgb]{1,0,0}{RED} highlights a \textit{strong} CDS in periodic phases (\ie, $\bm{\log_{10} \delta_P\geq -3.2}$), while \textcolor[rgb]{ 0,  .439,  .753}{BLUE} highlights a \textit{weak} CDS in periodic phases (\ie, $\bm{\log_{10} \delta_P< -3.2}$).}
    \resizebox{\linewidth}{!}{
    \begin{tabular}{c|c|cccc|cccc|cccc|cccc|cccc}
    \toprule
    \multicolumn{2}{c|}{\textbf{Dataset}} & \multicolumn{4}{c|}{\textbf{Illness}} & \multicolumn{4}{c|}{\textbf{Electricity}} & \multicolumn{4}{c|}{\textbf{Traffic}} & \multicolumn{4}{c|}{\textbf{ETTh1}} & \multicolumn{4}{c}{\textbf{ETTh2}} \\
    \midrule
    \multicolumn{2}{c|}{\textbf{Method}} & \multicolumn{2}{c}{\textbf{/}} & \multicolumn{2}{c|}{\textbf{+\method{SOLID}}} & \multicolumn{2}{c}{\textbf{/}} & \multicolumn{2}{c|}{\textbf{+\method{SOLID}}} & \multicolumn{2}{c}{\textbf{/}} & \multicolumn{2}{c|}{\textbf{+\method{SOLID}}} & \multicolumn{2}{c}{\textbf{/}} & \multicolumn{2}{c|}{\textbf{+\method{SOLID}}} & \multicolumn{2}{c}{\textbf{/}} & \multicolumn{2}{c}{\textbf{+\method{SOLID}}} \\
    \midrule
    \multicolumn{2}{c|}{\textbf{Metric}} & \textbf{MSE} & \textbf{MAE} & \textbf{MSE} & \textbf{MAE} & \textbf{MSE} & \textbf{MAE} & \textbf{MSE} & \textbf{MAE} & \textbf{MSE} & \textbf{MAE} & \textbf{MSE} & \textbf{MAE} & \textbf{MSE} & \textbf{MAE} & \textbf{MSE} & \textbf{MAE} & \textbf{MSE} & \textbf{MAE} & \textbf{MSE} & \textbf{MAE} \\
    \midrule
    \midrule
    \multirow{6}[2]{*}{\textbf{Informer}} & \textbf{24 / 96} & 5.096  & 1.533  & \textbf{2.874 } & \textbf{1.150 } & 0.321  & 0.407  & \textbf{0.245 } & \textbf{0.355 } & 0.731  & 0.406  & \textbf{0.628 } & \textbf{0.393 } & 0.948  & 0.774  & \textbf{0.684 } & \textbf{0.586 } & 2.992  & 1.362  & \textbf{1.659 } & \textbf{0.988 } \\
          & \textbf{36 / 192} & 5.078  & 1.535  & \textbf{3.299 } & \textbf{1.243 } & 0.351  & 0.434  & \textbf{0.256 } & \textbf{0.363 } & 0.739  & 0.414  & \textbf{0.653 } & \textbf{0.413 } & 1.009  & 0.786  & \textbf{0.759 } & \textbf{0.624 } & 6.256  & 2.091  & \textbf{3.564 } & \textbf{1.553 } \\
          & \textbf{48 / 336} & 5.144  & 1.567  & \textbf{2.879 } & \textbf{1.169 } & 0.349  & 0.432  & \textbf{0.279 } & \textbf{0.381 } & 0.850  & 0.476  & \textbf{0.758 } & \textbf{0.466 } & 1.035  & 0.783  & \textbf{0.776 } & \textbf{0.640 } & 5.265  & 1.954  & \textbf{1.845 } & \textbf{1.079 } \\
          & \textbf{60 / 720} & 5.243  & 1.582  & \textbf{3.773 } & \textbf{1.394 } & 0.385  & 0.493  & \textbf{0.327 } & \textbf{0.416 } & 0.945  & 0.530  & \textbf{0.861 } & \textbf{0.524 } & 1.153  & 0.845  & \textbf{0.948 } & \textbf{0.726 } & 4.038  & 1.673  & \textbf{2.020 } & \textbf{1.104 } \\
          & \textbf{$\uparrow$} & \textcolor[rgb]{ 1,  0,  0}{} & \textcolor[rgb]{ 1,  0,  0}{} & \textbf{37.62\%} & \textbf{20.29\%} & \textcolor[rgb]{ 1,  0,  0}{} & \textcolor[rgb]{ 1,  0,  0}{} & \textbf{21.28\%} & \textbf{14.21\%} & \textcolor[rgb]{ 1,  0,  0}{} & \textcolor[rgb]{ 1,  0,  0}{} & \textbf{11.19\%} & \textbf{1.65\%} & \textcolor[rgb]{ 1,  0,  0}{} & \textcolor[rgb]{ 1,  0,  0}{} & \textbf{23.60\%} & \textbf{19.18\%} & \textcolor[rgb]{ 1,  0,  0}{} & \textcolor[rgb]{ 1,  0,  0}{} & \textbf{51.01\%} & \textbf{33.29\%} \\
          & \textbf{$\bm \delta$} & \multicolumn{4}{c|}{\textcolor[rgb]{ 1,  0,  0}{\textbf{-1.096 \& -1.148}}} & \multicolumn{4}{c|}{\textcolor[rgb]{ 1,  0,  0}{\textbf{-2.975 \& -2.593}}} & \multicolumn{4}{c|}{\textcolor[rgb]{ 1,  0,  0}{\textbf{-2.762 \& -2.238}}} & \multicolumn{4}{c|}{\textcolor[rgb]{ 0,  .439,  .753}{\textbf{-3.83 \& -1.883}}} & \multicolumn{4}{c}{\textcolor[rgb]{ 1,  0,  0}{\textbf{-3.021 \& -1.513}}} \\
    \midrule
    \multirow{6}[2]{*}{\textbf{Autoformer}} & \textbf{24 / 96} & 3.314  & 1.245  & \textbf{2.737 } & \textbf{1.118 } & 0.207  & 0.324  & \textbf{0.189 } & \textbf{0.304 } & 0.621  & 0.391  & \textbf{0.565 } & \textbf{0.376 } & 0.440  & 0.444  & \textbf{0.430 } & \textbf{0.442 } & 0.363  & 0.405  & \textbf{0.362 } & \textbf{0.404 } \\
          & \textbf{36 / 192} & 2.733  & 1.078  & \textbf{2.540 } & \textbf{1.015 } & 0.221  & 0.334  & \textbf{0.205 } & \textbf{0.316 } & 0.666  & 0.415  & \textbf{0.599 } & \textbf{0.379 } & 0.487  & 0.472  & \textbf{0.480 } & \textbf{0.470 } & 0.450  & 0.447  & \textbf{0.449 } & \textbf{0.446 } \\
          & \textbf{48 / 336} & 2.651  & 1.075  & \textbf{2.455 } & \textbf{1.042 } & 0.244  & 0.350  & \textbf{0.232 } & \textbf{0.337 } & 0.649  & 0.405  & \textbf{0.595 } & \textbf{0.390 } & 0.471  & 0.475  & \textbf{0.467 } & \textbf{0.473 } & 0.470  & 0.474  & \textbf{0.468 } & \textbf{0.471 } \\
          & \textbf{60 / 720} & 2.848  & 1.126  & \textbf{2.689 } & \textbf{1.082 } & 0.285  & 0.381  & \textbf{0.273 } & \textbf{0.370 } & 0.684  & 0.422  & \textbf{0.635 } & \textbf{0.411 } & 0.543  & 0.528  & \textbf{0.542 } & \textbf{0.527 } & 0.484  & 0.491  & \textbf{0.483 } & \textbf{0.489 } \\
          & \textbf{$\uparrow$} & \textcolor[rgb]{ 1,  0,  0}{} & \textcolor[rgb]{ 1,  0,  0}{} & \textbf{9.75\%} & \textbf{5.90\%} & \textcolor[rgb]{ 1,  0,  0}{} & \textcolor[rgb]{ 1,  0,  0}{} & \textbf{6.16\%} & \textbf{4.44\%} & \textcolor[rgb]{ 1,  0,  0}{} & \textcolor[rgb]{ 1,  0,  0}{} & \textbf{8.61\%} & \textbf{4.80\%} & \textcolor[rgb]{ 1,  0,  0}{} & \textcolor[rgb]{ 1,  0,  0}{} & \textbf{1.08\%} & \textbf{0.35\%} & \textcolor[rgb]{ 1,  0,  0}{} & \textcolor[rgb]{ 1,  0,  0}{} & \textbf{0.25\%} & \textbf{0.35\%} \\
          & \textbf{$\bm \delta$} & \multicolumn{4}{c|}{\textcolor[rgb]{ 1,  0,  0}{\textbf{-1.166 \& -1.016}}} & \multicolumn{4}{c|}{\textcolor[rgb]{ 1,  0,  0}{\textbf{-2.408 \& -2.134}}} & \multicolumn{4}{c|}{\textcolor[rgb]{ 1,  0,  0}{\textbf{-2.507 \& -2.304}}} & \multicolumn{4}{c|}{\textcolor[rgb]{ 0,  .439,  .753}{\textbf{-3.572 \& -2.321}}} & \multicolumn{4}{c}{\textcolor[rgb]{ 0,  .439,  .753}{\textbf{-3.967 \& -1.921}}} \\
    \midrule
    \multirow{6}[2]{*}{\textbf{FEDformer}} & \textbf{24 / 96} & 3.241  & 1.252  & \textbf{2.707 } & \textbf{1.123 } & 0.188  & 0.304  & \textbf{0.172 } & \textbf{0.284 } & 0.574  & 0.356  & \textbf{0.513 } & \textbf{0.344 } & 0.375  & 0.414  & \textbf{0.370 } & \textbf{0.410 } & 0.341  & 0.385  & \textbf{0.339 } & \textbf{0.383 } \\
          & \textbf{36 / 192} & 2.576  & 1.048  & \textbf{2.365 } & \textbf{0.990 } & 0.197  & 0.311  & \textbf{0.180 } & \textbf{0.290 } & 0.612  & 0.379  & \textbf{0.549 } & \textbf{0.361 } & 0.427  & 0.448  & \textbf{0.420 } & \textbf{0.443 } & 0.433  & 0.441  & \textbf{0.432 } & \textbf{0.440 } \\
          & \textbf{48 / 336} & 2.546  & 1.058  & \textbf{2.435 } & \textbf{1.023 } & 0.213  & 0.328  & \textbf{0.195 } & \textbf{0.307 } & 0.618  & 0.379  & \textbf{0.557 } & \textbf{0.363 } & 0.458  & 0.465  & \textbf{0.454 } & \textbf{0.462 } & 0.503  & 0.494  & \textbf{0.501 } & \textbf{0.491 } \\
          & \textbf{60 / 720} & 2.784  & 1.136  & \textbf{2.677 } & \textbf{1.118 } & 0.243  & 0.352  & \textbf{0.228 } & \textbf{0.336 } & 0.629  & 0.382  & \textbf{0.577 } & \textbf{0.371 } & 0.482  & 0.495  & \textbf{0.478 } & \textbf{0.492 } & 0.479  & 0.485  & \textbf{0.478 } & \textbf{0.484 } \\
          & \textbf{$\uparrow$} & \textcolor[rgb]{ 1,  0,  0}{} & \textcolor[rgb]{ 1,  0,  0}{} & \textbf{8.64\%} & \textbf{5.34\%} & \textcolor[rgb]{ 1,  0,  0}{} & \textcolor[rgb]{ 1,  0,  0}{} & \textbf{7.88\%} & \textbf{5.95\%} & \textcolor[rgb]{ 1,  0,  0}{} & \textcolor[rgb]{ 1,  0,  0}{} & \textbf{9.77\%} & \textbf{3.70\%} & \textcolor[rgb]{ 1,  0,  0}{} & \textcolor[rgb]{ 1,  0,  0}{} & \textbf{1.10\%} & \textbf{0.74\%} & \textcolor[rgb]{ 1,  0,  0}{} & \textcolor[rgb]{ 1,  0,  0}{} & \textbf{0.42\%} & \textbf{0.40\%} \\
          & \textbf{$\bm \delta$} & \multicolumn{4}{c|}{\textcolor[rgb]{ 1,  0,  0}{\textbf{-1.099 \& -0.917}}} & \multicolumn{4}{c|}{\textcolor[rgb]{ 1,  0,  0}{\textbf{-2.967 \& -2.545}}} & \multicolumn{4}{c|}{\textcolor[rgb]{ 1,  0,  0}{\textbf{-2.254 \& -2.265}}} & \multicolumn{4}{c|}{\textcolor[rgb]{ 0,  .439,  .753}{\textbf{-3.52 \& -2.463}}} & \multicolumn{4}{c}{\textcolor[rgb]{ 0,  .439,  .753}{\textbf{-3.733 \& -1.943}}} \\
    \midrule
    \multirow{6}[2]{*}{\textbf{ETSformer}} & \textbf{24 / 96} & 2.397  & 0.993  & \textbf{2.262 } & \textbf{0.955 } & 0.187  & 0.304  & \textbf{0.171 } & \textbf{0.285 } & 0.599  & 0.386  & \textbf{0.487 } & \textbf{0.354 } & 0.495  & 0.480  & \textbf{0.491 } & \textbf{0.478 } & 0.346  & 0.401  & \textbf{0.344 } & \textbf{0.399 } \\
          & \textbf{36 / 192} & 2.504  & 0.970  & \textbf{2.301 } & \textbf{0.934 } & 0.198  & 0.313  & \textbf{0.184 } & \textbf{0.298 } & 0.611  & 0.391  & \textbf{0.492 } & \textbf{0.354 } & 0.543  & 0.505  & \textbf{0.538 } & \textbf{0.503 } & 0.437  & 0.447  & \textbf{0.430 } & \textbf{0.443 } \\
          & \textbf{48 / 336} & 2.488  & 0.999  & \textbf{2.320 } & \textbf{0.961 } & 0.210  & 0.326  & \textbf{0.197 } & \textbf{0.312 } & 0.619  & 0.393  & \textbf{0.501 } & \textbf{0.359 } & 0.581  & 0.521  & \textbf{0.574 } & \textbf{0.518 } & 0.478  & 0.479  & \textbf{0.467 } & \textbf{0.472 } \\
          & \textbf{60 / 720} & 2.494  & 1.011  & \textbf{2.358 } & \textbf{1.002 } & 0.249  & 0.356  & \textbf{0.234 } & \textbf{0.341 } & 0.629  & 0.391  & \textbf{0.529 } & \textbf{0.374 } & 0.569  & 0.534  & \textbf{0.562 } & \textbf{0.530 } & 0.488  & 0.492  & \textbf{0.474 } & \textbf{0.482 } \\
          & \textbf{$\uparrow$} & \textcolor[rgb]{ 1,  0,  0}{} & \textcolor[rgb]{ 1,  0,  0}{} & \textbf{6.50\%} & \textbf{3.03\%} & \textcolor[rgb]{ 1,  0,  0}{} & \textcolor[rgb]{ 1,  0,  0}{} & \textbf{6.89\%} & \textbf{4.75\%} & \textcolor[rgb]{ 1,  0,  0}{} & \textcolor[rgb]{ 1,  0,  0}{} & \textbf{18.28\%} & \textbf{7.57\%} & \textcolor[rgb]{ 1,  0,  0}{} & \textcolor[rgb]{ 1,  0,  0}{} & \textbf{1.08\%} & \textbf{0.56\%} & \textcolor[rgb]{ 1,  0,  0}{} & \textcolor[rgb]{ 1,  0,  0}{} & \textbf{1.93\%} & \textbf{1.31\%} \\
          & \textbf{$\bm \delta$} & \multicolumn{4}{c|}{\textcolor[rgb]{ 1,  0,  0}{\textbf{-1.544 \& -1.001}}} & \multicolumn{4}{c|}{\textcolor[rgb]{ 1,  0,  0}{\textbf{-2.559 \& -2.724}}} & \multicolumn{4}{c|}{\textcolor[rgb]{ 1,  0,  0}{\textbf{-2.488 \& -2.201}}} & \multicolumn{4}{c|}{\textcolor[rgb]{ 0,  .439,  .753}{\textbf{-3.207 \& -3.019}}} & \multicolumn{4}{c}{\textcolor[rgb]{ 1,  0,  0}{\textbf{-3.067 \& -1.079}}} \\
    \midrule
    \multirow{6}[2]{*}{\textbf{Crossformer}} & \textbf{24 / 96} & 3.329  & 1.275  & \textbf{2.353 } & \textbf{0.986 } & 0.184  & 0.297  & \textbf{0.182 } & \textbf{0.295 } & 0.521  & 0.297  & \textbf{0.473 } & \textbf{0.277 } & 0.411  & 0.432  & \textbf{0.382 } & \textbf{0.415 } & 0.641  & 0.555  & \textbf{0.527 } & \textbf{0.544 } \\
          & \textbf{36 / 192} & 3.392  & 1.185  & \textbf{2.527 } & \textbf{1.042 } & 0.219  & 0.317  & \textbf{0.216 } & \textbf{0.314 } & 0.523  & 0.298  & \textbf{0.475 } & \textbf{0.280 } & 0.419  & 0.444  & \textbf{0.396 } & \textbf{0.422 } & 1.262  & 0.814  & \textbf{0.834 } & \textbf{0.725 } \\
          & \textbf{48 / 336} & 3.481  & 1.228  & \textbf{2.499 } & \textbf{1.063 } & 0.238  & 0.348  & \textbf{0.235 } & \textbf{0.343 } & 0.530  & 0.300  & \textbf{0.481 } & \textbf{0.286 } & 0.439  & 0.459  & \textbf{0.418 } & \textbf{0.443 } & 1.486  & 0.896  & \textbf{1.048 } & \textbf{0.835 } \\
          & \textbf{60 / 720} & 3.571  & 1.234  & \textbf{3.103 } & \textbf{1.199 } & 0.274  & 0.373  & \textbf{0.269 } & \textbf{0.368 } & 0.573  & 0.313  & \textbf{0.498 } & \textbf{0.298 } & 0.504  & 0.514  & \textbf{0.473 } & \textbf{0.503 } & 1.220  & 0.848  & \textbf{0.906 } & \textbf{0.758 } \\
          & \textbf{$\uparrow$} & \textcolor[rgb]{ 1,  0,  0}{} & \textcolor[rgb]{ 1,  0,  0}{} & \textbf{23.89\%} & \textbf{12.84\%} & \textcolor[rgb]{ 1,  0,  0}{} & \textcolor[rgb]{ 1,  0,  0}{} & \textbf{1.42\%} & \textbf{1.11\%} & \textcolor[rgb]{ 1,  0,  0}{} & \textcolor[rgb]{ 1,  0,  0}{} & \textbf{10.25\%} & \textbf{5.55\%} & \textcolor[rgb]{ 1,  0,  0}{} & \textcolor[rgb]{ 1,  0,  0}{} & \textbf{5.94\%} & \textbf{3.53\%} & \textcolor[rgb]{ 1,  0,  0}{} & \textcolor[rgb]{ 1,  0,  0}{} & \textbf{28.05\%} & \textbf{8.06\%} \\
          & \textbf{$\bm \delta$} & \multicolumn{4}{c|}{\textcolor[rgb]{ 1,  0,  0}{\textbf{-1.038 \& -0.917}}} & \multicolumn{4}{c|}{\textcolor[rgb]{ 1,  0,  0}{\textbf{-2.333 \& -2.42}}} & \multicolumn{4}{c|}{\textcolor[rgb]{ 1,  0,  0}{\textbf{-2.879 \& -2.171}}} & \multicolumn{4}{c|}{\textcolor[rgb]{ 1,  0,  0}{\textbf{-3.111 \& -2.767}}} & \multicolumn{4}{c}{\textcolor[rgb]{ 1,  0,  0}{\textbf{-2.451 \& -1.149}}} \\
    \midrule
    \multirow{6}[2]{*}{\textbf{DLinear}} & \textbf{24 / 96} & 1.947  & 0.985  & \textbf{1.843 } & \textbf{0.938 } & 0.142  & 0.238  & \textbf{0.140 } & \textbf{0.237 } & 0.412  & 0.283  & \textbf{0.404 } & \textbf{0.277 } & 0.375  & 0.397  & \textbf{0.368 } & \textbf{0.391 } & 0.284  & 0.349  & \textbf{0.280 } & \textbf{0.346 } \\
          & \textbf{36 / 192} & 2.182  & 1.036  & \textbf{1.692 } & \textbf{0.898 } & 0.153  & 0.250  & \textbf{0.152 } & \textbf{0.249 } & 0.423  & 0.287  & \textbf{0.420 } & \textbf{0.285 } & 0.418  & 0.429  & \textbf{0.405 } & \textbf{0.416 } & 0.389  & 0.422  & \textbf{0.360 } & \textbf{0.398 } \\
          & \textbf{48 / 336} & 2.256  & 1.060  & \textbf{1.694 } & \textbf{0.916 } & 0.169  & 0.268  & \textbf{0.168 } & \textbf{0.266 } & 0.436  & 0.295  & \textbf{0.432 } & \textbf{0.291 } & 0.451  & 0.452  & \textbf{0.436 } & \textbf{0.435 } & 0.422  & 0.447  & \textbf{0.400 } & \textbf{0.430 } \\
          & \textbf{60 / 720} & 2.381  & 1.102  & \textbf{2.139 } & \textbf{1.012 } & 0.204  & 0.301  & \textbf{0.203 } & \textbf{0.300 } & 0.466  & 0.315  & \textbf{0.459 } & \textbf{0.310 } & 0.624  & 0.593  & \textbf{0.553 } & \textbf{0.547 } & 0.698  & 0.594  & \textbf{0.415 } & \textbf{0.451 } \\
          & \textbf{$\uparrow$} &       &       & \textbf{15.95\%} & \textbf{10.02\%} &       &       & \textbf{0.75\%} & \textbf{0.47\%} &       &       & \textbf{1.27\%} & \textbf{1.44\%} &       &       & \textbf{5.67\%} & \textbf{4.38\%} &       &       & \textbf{18.85\%} & \textbf{10.32\%} \\
          & \textbf{$\bm \delta$} & \multicolumn{4}{c|}{\textcolor[rgb]{ 1,  0,  0}{\textbf{-1.821 \& -1.301}}} & \multicolumn{4}{c|}{\textcolor[rgb]{ 0,  .439,  .753}{\textbf{-3.303 \& -2.645}}} & \multicolumn{4}{c|}{\textcolor[rgb]{ 1,  0,  0}{\textbf{-2.883 \& -2.589}}} & \multicolumn{4}{c|}{\textcolor[rgb]{ 1,  0,  0}{\textbf{-2.981 \& -2.321}}} & \multicolumn{4}{c}{\textcolor[rgb]{ 1,  0,  0}{\textbf{-3.023 \& -1.891}}} \\
    \midrule
    \multirow{6}[1]{*}{\textbf{PatchTST}} & \textbf{24 / 96} & 1.301  & 0.734  & \textbf{1.253 } & \textbf{0.710 } & 0.134  & 0.227  & \textbf{0.132 } & \textbf{0.226 } & 0.385  & 0.263  & \textbf{0.383 } & \textbf{0.262 } & 0.375  & 0.400  & \textbf{0.368 } & \textbf{0.393 } & 0.274  & 0.336  & \textbf{0.273 } & \textbf{0.336 } \\
          & \textbf{36 / 192} & 1.483  & 0.841  & \textbf{1.449 } & \textbf{0.823 } & 0.151  & 0.243  & \textbf{0.150 } & \textbf{0.242 } & 0.393  & 0.265  & \textbf{0.392 } & \textbf{0.264 } & 0.408  & 0.411  & \textbf{0.402 } & \textbf{0.407 } & 0.340  & 0.380  & \textbf{0.339 } & \textbf{0.379 } \\
          & \textbf{48 / 336} & 1.652  & 0.845  & \textbf{1.624 } & \textbf{0.831 } & 0.168  & 0.262  & \textbf{0.167 } & \textbf{0.262 } & 0.403  & 0.273  & \textbf{0.402 } & \textbf{0.271 } & 0.431  & 0.430  & \textbf{0.428 } & \textbf{0.428 } & 0.332  & 0.383  & \textbf{0.331 } & \textbf{0.382 } \\
          & \textbf{60 / 720} & 1.731  & 0.886  & \textbf{1.661 } & \textbf{0.843 } & 0.201  & 0.292  & \textbf{0.200 } & \textbf{0.291 } & 0.439  & 0.295  & \textbf{0.438 } & \textbf{0.295 } & 0.442  & 0.452  & \textbf{0.433 } & \textbf{0.447 } & 0.378  & 0.420  & \textbf{0.377 } & \textbf{0.418 } \\
          & \textbf{$\uparrow$} &       &       & \textbf{2.92\%} & \textbf{2.99\%} &       &       & \textbf{0.76\%} & \textbf{0.29\%} &       &       & \textbf{0.31\%} & \textbf{0.36\%} &       &       & \textbf{1.51\%} & \textbf{1.06\%} &       &       & \textbf{0.30\%} & \textbf{0.26\%} \\
          & \textbf{$\bm \delta$} & \multicolumn{4}{c|}{\textcolor[rgb]{ 1,  0,  0}{\textbf{-1.172 \& -1.054}}} & \multicolumn{4}{c|}{\textcolor[rgb]{ 0,  .439,  .753}{\textbf{-3.834 \& -2.878}}} & \multicolumn{4}{c|}{\textcolor[rgb]{ 0,  .439,  .753}{\textbf{-3.677 \& -2.901}}} & \multicolumn{4}{c|}{\textcolor[rgb]{ 1,  0,  0}{\textbf{-3.087 \& -2.529}}} & \multicolumn{4}{c}{\textcolor[rgb]{ 0,  .439,  .753}{\textbf{-3.299 \& -1.851}}} \\
    \midrule
    \bottomrule
    \end{tabular}%
    }
        \begin{tablenotes}
      \small
      \item Results of \textit{ETTm1}, \textit{ETTm2} and \textit{Weather} datasets are included in Table \ref{tbl:Main_Results_Part2} of Appendix \ref{sec:appendix_full_result}, due to space limit.
    \end{tablenotes}
  \label{tbl:Main_Results_Part1}
\end{table*}%

%% file: sec-conclusion.tex
\section{Conclusion}
\par In this paper, we introduce context-driven distribution shift (CDS) problem in TSF and identify two significant observed contexts, including temporal segments and periodic phases, along with unobserved contexts. To address the issue, we propose a general calibration framework, including a detector, \method{Reconditionor}, to evaluate the degree of a model’s susceptibility to CDS and the necessity for model adaptation; and an adaptation framework, \method{SOLID}, for calibrating models and enhancing their performance under severe CDS. We conduct extensive experiments on 8 real-world datasets and 7 models, demonstrating the accuracy of \method{Reconditionor} in detecting CDS and the effectiveness of \method{SOLID} in adapting TSF models without substantially compromising time efficiency. This adaptation consistently leads to improved performance, which can be well explained by \method{Reconditionor}.

\stitle{Limitations and future work.} (1) Despite demonstrating consistent performance improvements, \method{SOLID} introduces an approximate additional 20\% time overhead during the testing phase. Employing further engineering optimizations to narrow this efficiency gap will be investigated in future work. (2) While \method{Reconditionor} proves effective in leveraging observed contexts to ascertain the impact of CDS, whether unobserved context can be used for detection is an interesting open question. (3) The selection of contextually similar samples relies on heuristic rules in this paper. Investigating further contexts or developing more generalized selection criteria represents a promising area of research.

\begin{acks}
Research work mentioned in this paper is supported by State Street Zhejiang University Technology Center. We would also like to thank Yusu Hong for the valuable discussion of the theory.
\end{acks}

%% file: sec-appendix-kdd.tex
\appendix

\numberwithin{equation}{section}
\section*{Appendix}

\section{Proofs for the theoretical results}\label{sec:appendix_proof}

Proofs for Theorem \ref{thm:error-global} and Theorem \ref{thm:error-context} can be found in our full version paper\footnote{\url{https://arxiv.org/abs/2310.14838}}.

\section{Details of experiments} 

\subsection{Hyper-parameters} \label{hyper-parameters}
\par For all experiments, during the training process, we use the same hyper-parameters as reported in the corresponding papers \cite{Informer, Autoformer, FEDformer, ETSformer, Crossformer, LTSF-Linear}, \eg, encoder/decoder layers, model hidden dimensions, head numbers of multi-head attention and batch size. 
\par As for hyper-parameters for adaptation process, there exists 4 major hyper-parameters for \method{SOLID}, including $\lambda_T$, $\lambda_P$, $\lambda_N$, $lr$ (We report the ratio $\nicefrac{lr}{lr_{\text{train}}}$ between adaptation learning rate $lr$ and the training learning rate $lr_{\text{train}}$ as an alternative).
We select the setting which performs the best on the validation set. The search range for the parameters is presented in Table \ref{tbl:grid-search}.

\input{table/param_range}

\subsection{Specific structures of $g_\Phi$ and $h_\theta$} \label{specific_structure_of_different_models}
\par We divide our baseline TSF models into three categories, including Encode-decoder-based, Encoder-based, and Linear-based architectures for explanation.

\begin{itemize}[leftmargin=*]
    \item Encode-decoder-based (\eg, Informer, Autoformer, FEDformer, ETSformer, and Crossformer): $h_\theta$ refers to the linear projection layer at the top of decoders, and the other parts of decoders and the entire encoders correspond to $g_\Phi$.
    \item Encoder-based (\eg, PatchTST): $h_\theta$ refers to the top linear prediction layer at the top of the encoders, and the other parts (including self-attention and embedding layers) correspond to $g_\Phi$.
    \item Linear-based (\eg, DLinear): Since Linear models only include linear layers for modeling, $h_\theta$ refers to the linear layers, and there is no $g_\Phi$.
\end{itemize}

\section{Further Experiment Results}

\subsection{Full results}\label{sec:appendix_full_result}

In this section, we report the full results of \sref{sec:main_result}. Table \ref{tbl:Main_Results_Part2} details the benchmark results for ETTm1, ETTm2, and Weather datasets, where our proposed \method{Reconditioner} identified relatively lower metrics.

\input{table/extended_main_results}


\begin{figure}[t]
    \centering  
    \subfigure[context: temporal segments]{
        \label{fig:intro-confounder-time}
        \includegraphics[width=0.22\textwidth,trim={10 10 10 10},clip]{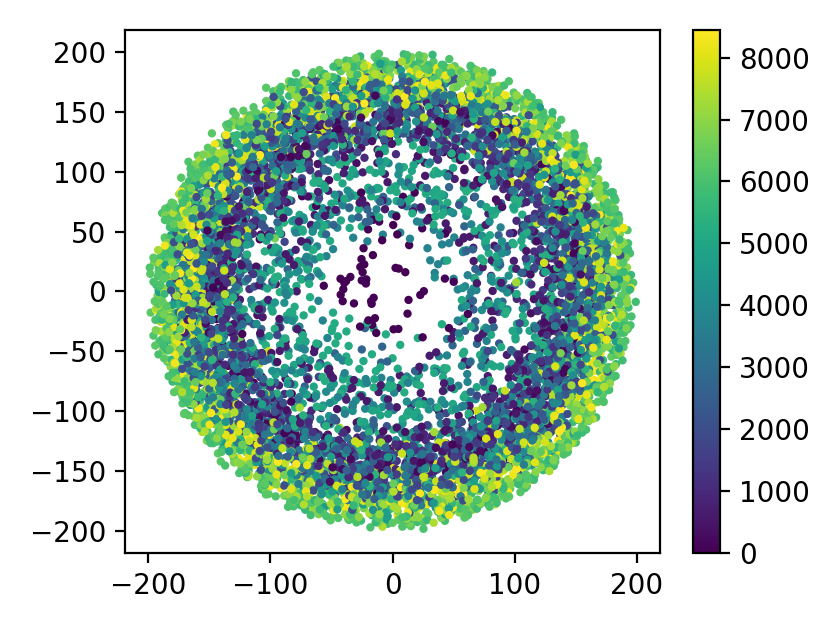}
    }
    \subfigure[context: periodic phases]{
        \label{fig:intro-confounder-phases}
        \includegraphics[width=0.22\textwidth,trim={10 10 10 10},clip]{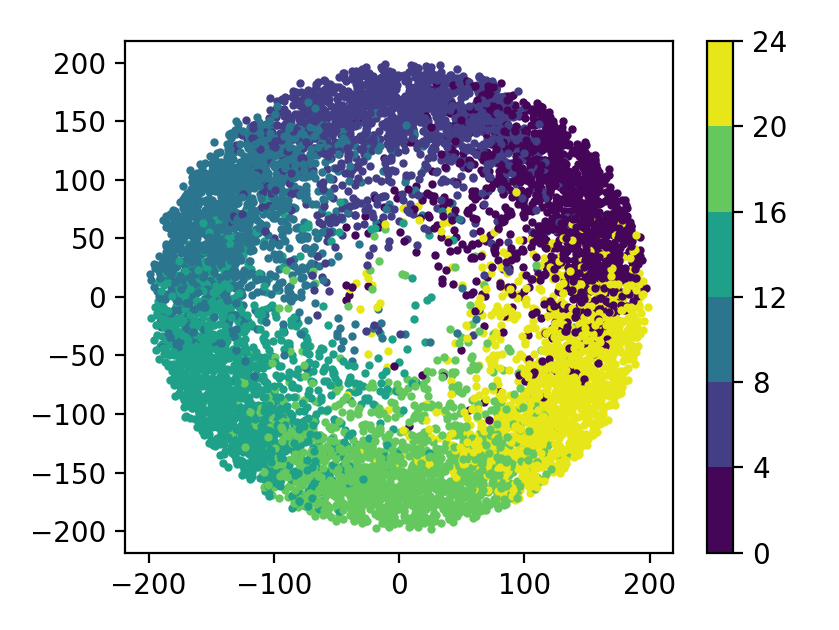}
    }
    \caption{Illustration for two contexts on ETTh1 training dataset, which has a period of 24 and length of 8760. Different colors represent different values of contexts.}
    \label{fig:intro-confounder}
\end{figure}

\subsection{Visualizing the impact of contexts on data distribution}\label{sec:appendix_context_visual}

To demonstrate the influence of context on data distribution, we utilized the Autoformer \cite{Autoformer} to extract latent representations from the ETTh1 dataset and applied PCA for dimensionality reduction and visualization. This ensures the spatial positions of the data points in the figure represent their original distribution. 

We marked two observed contexts -- temporal segments and periodic phases -- on the figures. Unobserved contexts, being difficult to visualize, are not displayed. Figure \ref{fig:intro-confounder-time} reveals a progressive outward shift in data distribution with increasing temporal segments. Similarly, Figure \ref{fig:intro-confounder-phases} shows that changes in the periodic phase lead to a rotational shift in data distribution. Note that if CDS doesn't exist, different colors (denoting contexts) should be scattered and randomly mixed in the figure since colors are independent of spatial positions, which is not the case shown in Figure \ref{fig:intro-confounder}. Therefore, it is evident that these contexts markedly affect data distribution.

\subsection{Correlation of $\delta_T$ and MAE improvement}\label{sec:app_corr_t}

\begin{figure}[t]
    \centering
    \includegraphics[width=0.47\textwidth,trim={10 10 10 10},clip]{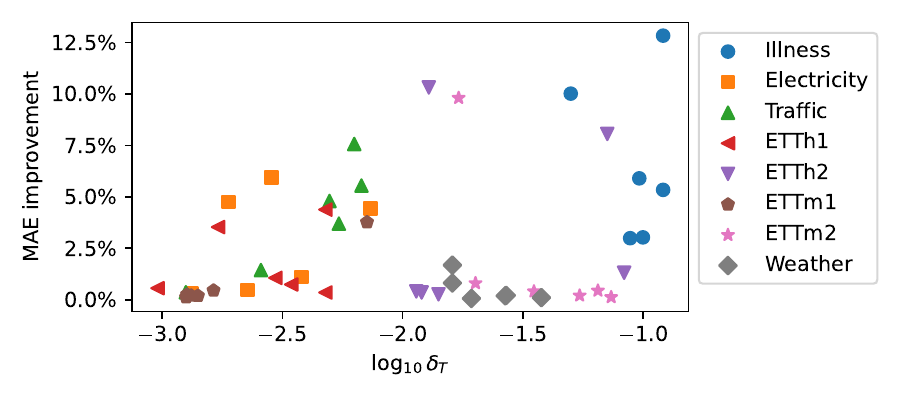}
    \caption{The relationship $\bm{\log_{10} \delta_T}$ (X-axis) and MAE improvements achieved by \method{SOLID} (Y-axis) for 8 datasets and 6 models. }
    \label{fig:indicator-time}
\end{figure}

In \sref{sec:relation_reconditionor_solid}, we found that the correlation between $\delta_P$ and MAE improvement is very strong. We display the correlation between $\delta_T$ and MAE improvement in Figure \ref{fig:indicator-time} as well. Unlike $\delta_P$, the relation between $\delta_T$ and MAE improvement is not as straightforward. Despite showing an increasing trend in Figure \ref{fig:indicator-time}, the Spearman correlation coefficient is merely 0.2129. This implies that, while there is evidence of CDS stemming from temporal segments as detected by \method{Reconditionor}, \method{SOLID} is comparatively less effective at mitigating it when compared to CDS caused by periodic phases. One possible explanation is that data generated within the same phase tends to follow a more predictable pattern, while data within the same temporal segment exhibits greater diversity and uncertainty, which may limit the utility of selecting data from the same segment to address CDS caused by temporal segments. We leave further investigation of it to future work.

\input{table/param_experiment}
\input{table/case_study_short}

\subsection{Parameter sensitivity analysis} \label{sec:appendix-parameter-sensitivity}

\par Within our proposed approach, several crucial parameters are presented, including $\lambda_T$, which controls the time range of preceding data for selection; $\lambda_P$, which governs the acceptable threshold for periodic phase difference; $\lambda_N$, which determines the number of similar samples to be selected for model adaptation; And $lr$, which regulates the extent of adaptation on the prediction layer for the models. The search range for these parameters is presented in Table \ref{tbl:grid-search}. The results of parameter sensitivity analysis are visually presented in Figure \ref{fig:parameter_sensitivity}.

\par Firstly, we discover that our proposed \method{SOLID} is insensitive to $\lambda_T$ and $\lambda_P$ parameters, based on the results obtained.
Regarding the $\lambda_N$ parameter, the selection of insufficient samples would increase the variance during adaptation due to the data shortage. Conversely, selecting excessive samples carries the risk of including samples with unobserved irrelevant contexts, thereby deteriorating the model’s performance.
Lastly, For parameter $lr$, a very small $lr$ leads to inadequate model adaptation, preventing the model from effectively addressing CDS and resulting in a bias towards the test sample (Theorem \ref{thm:error-global}). Conversely, a too large value for $lr$ can lead to excessive adaptation, which also risks bringing in substantial variance to the model (Theorem \ref{thm:error-context}). Therefore, a well-selected learning rate will contribute to an optimal trade-off between bias and variance. 

\subsection{Case study} \label{sec:appendix-visulization}

\par In addition, we conduct a case study and visualize our proposed method on various cases across different datasets and models, as presented in Figure \ref{fig:visualization-of-prediction}. 
\par Specifically, we plot the figures on these combinations: Illness (timestep 120~\&~160, variate 7) on Crossformer, and ETTh1 (timestep 250~\&~650, variate 7) on Informer, corresponding to Figure \ref{fig:visualization-of-prediction} (a)-(d).
The visualization vividly illustrates the effectiveness of our approach in improving forecasting performance.

%% file: table/param_range.tex
\newcommand{\rotaterange}[1]{\begin{tabular}{@{}c@{}}\multirow{10}{*}{{\rotatebox[origin=c]{90}{#1}}}\end{tabular}}

\begin{table}[htbp]
  \centering
  \caption{Hyper-parameters for \method{SOLID}.}
  \resizebox{\linewidth}{!}{
    \begin{tabular}{c|c|c|c|c}
    \hline
    \textbf{Dataset} & \textbf{$\lambda_T$} & \textbf{$\lambda_P$} & \textbf{$\lambda_N$} & \textbf{$\nicefrac{lr}{lr_{\text{train}}}$} \\
    \hline
    \hline
    \textbf{ETTh1} & \multirow{7}[14]{*}{$\{500, 1000, 2000\}$} & \rotaterange{$ \{0.02, 0.05, 0.1\}$} & \multirow{7}[14]{*}{$\{5, 10, 20\}$} & \multirow{4}[8]{*}{$\{5, 10, 20, 50\}$} \bigstrut\\
\cline{1-1}    \textbf{ETTh2} &       &       &       & \bigstrut\\
\cline{1-1}    \textbf{ETTm1} &       &       &       & \bigstrut\\
\cline{1-1}    \textbf{ETTm2} &       &       &       & \bigstrut\\
\cline{1-1}\cline{5-5}    \textbf{Electricity} &       &       &       & $\{500, 1000, 1500, 2000\}$  \bigstrut\\
\cline{1-1}\cline{5-5}    \textbf{Traffic} &       &       &       & $\{1000, 1500, 2000, 3000\}$  \bigstrut\\
\cline{1-1}\cline{5-5}    \textbf{Weather} &       &       &       & $\{5, 10, 20, 50\}$  \bigstrut\\
\cline{1-1}\cline{4-5}    \textbf{Illness} & $\{100, 200, 300\}$ &       & $\{2, 3, 5\}$ & $\{10, 20, 50, 100\}$ 
  \bigstrut\\
    \hline
    \end{tabular}%
  }
  \label{tbl:grid-search}%
\end{table}%

%% file: table/extended_main_results.tex
\newcommand{\rotatemodel}[1]{\begin{tabular}{@{}c@{}}\multirow{6}{*}{{\rotatebox[origin=c]{90}{#1}}}\end{tabular}}

\begin{table}[htbp]
  \centering
  \small
  \caption{Extended table of Table \ref{tbl:Main_Results_Part1}. "$\bm \uparrow$": average improvements achieved by \method{\textit{SOLID}} compared to the original forecasting results. "$\bm \delta$": metrics given by \method{\textit{Reconditionor}} in two observed contexts: periodic phases ($\bm{\delta_P}$) and temporal segments ($\bm{\delta_T}$), presented in the form of $\bm{\log_{10} \delta_P\ \&\ \log_{10} \delta_T}$. \textcolor[rgb]{1,0,0}{RED} highlights a \textit{strong} CDS in periodic phases (\ie, $\bm{\log_{10} \delta_P\geq -3.2}$), while \textcolor[rgb]{ 0,  .439,  .753}{BLUE} highlights a \textit{weak} CDS in periodic phases (\ie, $\bm{\log_{10} \delta_P< -3.2}$).}
  \resizebox{\linewidth}{!}{
    \begin{tabular}{c|c|cccc|cccc|cccc}
    \toprule
    \multicolumn{2}{c|}{\textbf{Dataset}} & \multicolumn{4}{c|}{\textbf{ETTm1}} & \multicolumn{4}{c|}{\textbf{ETTm2}} & \multicolumn{4}{c}{\textbf{Weather}} \\
    \midrule
    \multicolumn{2}{c|}{\textbf{Method}} & \multicolumn{2}{c}{\textbf{/}} & \multicolumn{2}{c|}{\textbf{+\method{SOLID}}} & \multicolumn{2}{c}{\textbf{/}} & \multicolumn{2}{c|}{\textbf{+\method{SOLID}}} & \multicolumn{2}{c}{\textbf{/}} & \multicolumn{2}{c}{\textbf{+\method{SOLID}}} \\
    \midrule
    \multicolumn{2}{c|}{\textbf{Metric}} & \textbf{MSE} & \textbf{MAE} & \textbf{MSE} & \textbf{MAE} & \textbf{MSE} & \textbf{MAE} & \textbf{MSE} & \textbf{MAE} & \textbf{MSE} & \textbf{MAE} & \textbf{MSE} & \textbf{MAE} \\
    \midrule
    \midrule
    \rotatemodel{\textbf{Informer}} & \textbf{96} & 0.624  & 0.556  & \textbf{0.426 } & \textbf{0.436 } & 0.412  & 0.498  & \textbf{0.279 } & \textbf{0.379 } & 0.394  & 0.436  & \textbf{0.321 } & \textbf{0.377 } \\
          & \textbf{192} & 0.727  & 0.620  & \textbf{0.535 } & \textbf{0.503 } & 0.821  & 0.710  & \textbf{0.453 } & \textbf{0.503 } & 0.501  & 0.491  & \textbf{0.330 } & \textbf{0.367 } \\
          & \textbf{336} & 1.085  & 0.776  & \textbf{0.696 } & \textbf{0.601 } & 1.459  & 0.926  & \textbf{0.664 } & \textbf{0.612 } & 0.591  & 0.540  & \textbf{0.400 } & \textbf{0.421 } \\
          & \textbf{720} & 1.200  & 0.814  & \textbf{0.775 } & \textbf{0.631 } & 3.870  & 1.461  & \textbf{1.875 } & \textbf{0.982 } & 1.070  & 0.755  & \textbf{0.681 } & \textbf{0.576 } \\
          & \textbf{$\uparrow$} & \textcolor[rgb]{ 1,  0,  0}{} & \textcolor[rgb]{ 1,  0,  0}{} & \textbf{33.13\%} & \textbf{21.52\%} & \textcolor[rgb]{ 1,  0,  0}{} & \textcolor[rgb]{ 1,  0,  0}{} & \textbf{50.15\%} & \textbf{31.11\%} & \textcolor[rgb]{ 1,  0,  0}{} & \textcolor[rgb]{ 1,  0,  0}{} & \textbf{32.19\%} & \textbf{21.71\%} \\
          & \textbf{$\bm \delta$} & \multicolumn{4}{c|}{\textcolor[rgb]{ 0,  .439,  .753}{\textbf{-3.357 \& -2.229}}} & \multicolumn{4}{c|}{\textcolor[rgb]{ 1,  0,  0}{\textbf{-3.082 \& -1.527}}} & \multicolumn{4}{c}{\textcolor[rgb]{ 0,  .439,  .753}{\textbf{-4.495 \& -1.402}}} \\
    \midrule
    \rotatemodel{\textbf{Autoformer}} & \textbf{96} & 0.445  & 0.449  & \textbf{0.443 } & \textbf{0.448 } & 0.313  & 0.346  & \textbf{0.312 } & \textbf{0.345 } & 0.260  & 0.334  & \textbf{0.259 } & \textbf{0.334 } \\
          & \textbf{192} & 0.549  & 0.500  & \textbf{0.547 } & \textbf{0.499 } & 0.283  & 0.341  & \textbf{0.282 } & \textbf{0.340 } & 0.322  & 0.377  & \textbf{0.321 } & \textbf{0.376 } \\
          & \textbf{336} & 0.633  & 0.533  & \textbf{0.631 } & \textbf{0.532 } & 0.327  & 0.367  & \textbf{0.326 } & \textbf{0.366 } & 0.367  & 0.398  & \textbf{0.367 } & \textbf{0.397 } \\
          & \textbf{720} & 0.672  & 0.559  & \textbf{0.670 } & \textbf{0.558 } & 0.445  & 0.434  & \textbf{0.444 } & \textbf{0.433 } & 0.414  & 0.421  & \textbf{0.414 } & \textbf{0.421 } \\
          & \textbf{$\uparrow$} & \textcolor[rgb]{ 1,  0,  0}{} & \textcolor[rgb]{ 1,  0,  0}{} & \textbf{0.30\%} & \textbf{0.19\%} & \textcolor[rgb]{ 1,  0,  0}{} & \textcolor[rgb]{ 1,  0,  0}{} & \textbf{0.24\%} & \textbf{0.20\%} & \textcolor[rgb]{ 1,  0,  0}{} & \textcolor[rgb]{ 1,  0,  0}{} & \textbf{0.06\%} & \textbf{0.06\%} \\
          & \textbf{$\bm \delta$} & \multicolumn{4}{c|}{\textcolor[rgb]{ 0,  .439,  .753}{\textbf{-3.851 \& -2.851}}} & \multicolumn{4}{c|}{\textcolor[rgb]{ 0,  .439,  .753}{\textbf{-3.524 \& -1.264}}} & \multicolumn{4}{c}{\textcolor[rgb]{ 0,  .439,  .753}{\textbf{-4.236 \& -1.714}}} \\
    \midrule
    \rotatemodel{\textbf{FEDformer}} & \textbf{96} & 0.362  & 0.412  & \textbf{0.360 } & \textbf{0.410 } & 0.190  & 0.283  & \textbf{0.189 } & \textbf{0.283 } & 0.225  & 0.307  & \textbf{0.225 } & \textbf{0.307 } \\
          & \textbf{192} & 0.399  & 0.428  & \textbf{0.398 } & \textbf{0.427 } & 0.256  & 0.324  & \textbf{0.255 } & \textbf{0.324 } & 0.318  & 0.374  & \textbf{0.317 } & \textbf{0.374 } \\
          & \textbf{336} & 0.441  & 0.456  & \textbf{0.440 } & \textbf{0.455 } & 0.327  & 0.365  & \textbf{0.326 } & \textbf{0.364 } & 0.347  & 0.385  & \textbf{0.346 } & \textbf{0.385 } \\
          & \textbf{720} & 0.486  & 0.477  & \textbf{0.485 } & \textbf{0.477 } & 0.434  & 0.425  & \textbf{0.434 } & \textbf{0.424 } & 0.408  & 0.424  & \textbf{0.407 } & \textbf{0.422 } \\
          & \textbf{$\uparrow$} & \textcolor[rgb]{ 1,  0,  0}{} & \textcolor[rgb]{ 1,  0,  0}{} & \textbf{0.31\%} & \textbf{0.25\%} & \textcolor[rgb]{ 1,  0,  0}{} & \textcolor[rgb]{ 1,  0,  0}{} & \textbf{0.14\%} & \textbf{0.12\%} & \textcolor[rgb]{ 1,  0,  0}{} & \textcolor[rgb]{ 1,  0,  0}{} & \textbf{0.22\%} & \textbf{0.21\%} \\
          & \textbf{$\bm \delta$} & \multicolumn{4}{c|}{\textcolor[rgb]{ 0,  .439,  .753}{\textbf{-3.979 \& -2.896}}} & \multicolumn{4}{c|}{\textcolor[rgb]{ 0,  .439,  .753}{\textbf{-3.914 \& -1.133}}} & \multicolumn{4}{c}{\textcolor[rgb]{ 0,  .439,  .753}{\textbf{-4.159 \& -1.569}}} \\
    \midrule
    \rotatemodel{\textbf{ETSformer}} & \textbf{96} & 0.371  & 0.394  & \textbf{0.369 } & \textbf{0.392 } & 0.187  & 0.280  & \textbf{0.186 } & \textbf{0.279 } & 0.216  & 0.298  & \textbf{0.213 } & \textbf{0.295 } \\
          & \textbf{192} & 0.402  & 0.405  & \textbf{0.402 } & \textbf{0.405 } & 0.251  & 0.319  & \textbf{0.251 } & \textbf{0.318 } & 0.253  & 0.329  & \textbf{0.251 } & \textbf{0.326 } \\
          & \textbf{336} & 0.429  & 0.423  & \textbf{0.429 } & \textbf{0.423 } & 0.313  & 0.356  & \textbf{0.312 } & \textbf{0.355 } & 0.289  & 0.352  & \textbf{0.286 } & \textbf{0.348 } \\
          & \textbf{720} & 0.429  & 0.455  & \textbf{0.428 } & \textbf{0.455 } & 0.414  & 0.413  & \textbf{0.410 } & \textbf{0.410 } & 0.355  & 0.401  & \textbf{0.353 } & \textbf{0.399 } \\
          & \textbf{$\uparrow$} & \textcolor[rgb]{ 1,  0,  0}{} & \textcolor[rgb]{ 1,  0,  0}{} & \textbf{0.14\%} & \textbf{0.12\%} & \textcolor[rgb]{ 1,  0,  0}{} & \textcolor[rgb]{ 1,  0,  0}{} & \textbf{0.56\%} & \textbf{0.44\%} & \textcolor[rgb]{ 1,  0,  0}{} & \textcolor[rgb]{ 1,  0,  0}{} & \textbf{0.93\%} & \textbf{0.81\%} \\
          & \textbf{$\bm \delta$} & \multicolumn{4}{c|}{\textcolor[rgb]{ 0,  .439,  .753}{\textbf{-3.996 \& -2.900}}} & \multicolumn{4}{c|}{\textcolor[rgb]{ 0,  .439,  .753}{\textbf{-3.541 \& -1.188}}} & \multicolumn{4}{c}{\textcolor[rgb]{ 0,  .439,  .753}{\textbf{-4.614 \& -1.793}}} \\
    \midrule
    \rotatemodel{\textbf{Crossformer}} & \textbf{96} & 0.312  & 0.367  & \textbf{0.311 } & \textbf{0.367 } & 0.770  & 0.599  & \textbf{0.689 } & \textbf{0.580 } & 0.151  & 0.219  & \textbf{0.151 } & \textbf{0.219 } \\
          & \textbf{192} & 0.350  & 0.391  & \textbf{0.348 } & \textbf{0.390 } & 0.567  & 0.516  & \textbf{0.450 } & \textbf{0.512 } & 0.196  & 0.264  & \textbf{0.196 } & \textbf{0.264 } \\
          & \textbf{336} & 0.407  & 0.427  & \textbf{0.403 } & \textbf{0.425 } & 0.830  & 0.637  & \textbf{0.616 } & \textbf{0.613 } & 0.246  & 0.307  & \textbf{0.246 } & \textbf{0.306 } \\
          & \textbf{720} & 0.648  & 0.580  & \textbf{0.538 } & \textbf{0.515 } & 1.754  & 1.010  & \textbf{1.037 } & \textbf{0.786 } & 0.311  & 0.357  & \textbf{0.311 } & \textbf{0.357 } \\
          & \textbf{$\uparrow$} & \textcolor[rgb]{ 1,  0,  0}{} & \textcolor[rgb]{ 1,  0,  0}{} & \textbf{6.81\%} & \textbf{3.77\%} & \textcolor[rgb]{ 1,  0,  0}{} & \textcolor[rgb]{ 1,  0,  0}{} & \textbf{28.75\%} & \textbf{9.81\%} & \textcolor[rgb]{ 1,  0,  0}{} & \textcolor[rgb]{ 1,  0,  0}{} & \textbf{0.11\%} & \textbf{0.10\%} \\
          & \textbf{$\bm \delta$} & \multicolumn{4}{c|}{\textcolor[rgb]{ 0,  .439,  .753}{\textbf{-3.369 \& -2.148}}} & \multicolumn{4}{c|}{\textcolor[rgb]{ 1,  0,  0}{\textbf{-2.979 \& -1.767}}} & \multicolumn{4}{c}{\textcolor[rgb]{ 0,  .439,  .753}{\textbf{-4.521 \& -1.423}}} \\
    \midrule
    \rotatemodel{\textbf{DLinear}} & \textbf{96} & 0.310  & 0.353  & \textbf{0.306 } & \textbf{0.349 } & 0.169  & 0.262  & \textbf{0.168 } & \textbf{0.261 } & 0.176  & 0.238  & \textbf{0.175 } & \textbf{0.232 } \\
          & \textbf{192} & 0.340  & 0.369  & \textbf{0.339 } & \textbf{0.368 } & 0.232  & 0.308  & \textbf{0.230 } & \textbf{0.306 } & 0.218  & 0.276  & \textbf{0.216 } & \textbf{0.272 } \\
          & \textbf{336} & 0.374  & 0.390  & \textbf{0.373 } & \textbf{0.389 } & 0.299  & 0.360  & \textbf{0.296 } & \textbf{0.356 } & 0.262  & 0.312  & \textbf{0.261 } & \textbf{0.308 } \\
          & \textbf{720} & 0.440  & 0.435  & \textbf{0.437 } & \textbf{0.434 } & 0.439  & 0.451  & \textbf{0.434 } & \textbf{0.447 } & 0.326  & 0.365  & \textbf{0.324 } & \textbf{0.359 } \\
          & \textbf{$\uparrow$} &       &       & \textbf{0.61\%} & \textbf{0.45\%} &       &       & \textbf{0.97\%} & \textbf{0.80\%} &       &       & \textbf{0.61\%} & \textbf{1.68\%} \\
          & \textbf{$\bm \delta$} & \multicolumn{4}{c|}{\textcolor[rgb]{ 1,  0,  0}{\textbf{-3.057 \& -2.786}}} & \multicolumn{4}{c|}{\textcolor[rgb]{ 0,  .439,  .753}{\textbf{-3.521 \& -1.697}}} & \multicolumn{4}{c}{\textcolor[rgb]{ 0,  .439,  .753}{\textbf{-4.386 \& -1.793}}} \\
    \midrule
    \rotatemodel{\textbf{PatchTST}} & \textbf{96} & 0.292  & 0.345  & \textbf{0.290 } & \textbf{0.344 } & 0.166  & 0.257  & \textbf{0.165 } & \textbf{0.255 } & 0.152  & 0.200  & \textbf{0.151 } & \textbf{0.199 } \\
          & \textbf{192} & 0.333  & 0.370  & \textbf{0.332 } & \textbf{0.370 } & 0.220  & 0.293  & \textbf{0.219 } & \textbf{0.292 } & 0.197  & 0.243  & \textbf{0.196 } & \textbf{0.242 } \\
          & \textbf{336} & 0.366  & 0.390  & \textbf{0.365 } & \textbf{0.389 } & 0.275  & 0.329  & \textbf{0.274 } & \textbf{0.328 } & 0.250  & 0.285  & \textbf{0.250 } & \textbf{0.285 } \\
          & \textbf{720} & 0.420  & 0.424  & \textbf{0.419 } & \textbf{0.423 } & 0.366  & 0.385  & \textbf{0.365 } & \textbf{0.384 } & 0.316  & 0.334  & \textbf{0.316 } & \textbf{0.334 } \\
          & \textbf{$\uparrow$} &       &       & \textbf{0.35\%} & \textbf{0.20\%} &       &       & \textbf{0.39\%} & \textbf{0.40\%} &       &       & \textbf{0.22\%} & \textbf{0.19\%} \\
          & \textbf{$\bm \delta$} & \multicolumn{4}{c|}{\textcolor[rgb]{ 0,  .439,  .753}{\textbf{-3.356 \& -2.879}}} & \multicolumn{4}{c|}{\textcolor[rgb]{ 0,  .439,  .753}{\textbf{-3.471 \& -1.454}}} & \multicolumn{4}{c}{\textcolor[rgb]{ 0,  .439,  .753}{\textbf{-3.986 \& -1.573}}} \\
    \midrule
    \bottomrule
    \end{tabular}%
  }
  \label{tbl:Main_Results_Part2}%
\end{table}%

%% file: table/param_experiment.tex
\begin{figure*}[t!]
    \centering  
    \caption{Parameter sensitivity results for $\lambda_T$, $\lambda_P$, $\lambda_N$ and $lr$, on Illness and Traffic datasets.}
    \subfigure[Illness-$\lambda_T$]{
        \includegraphics[width=0.24\textwidth]{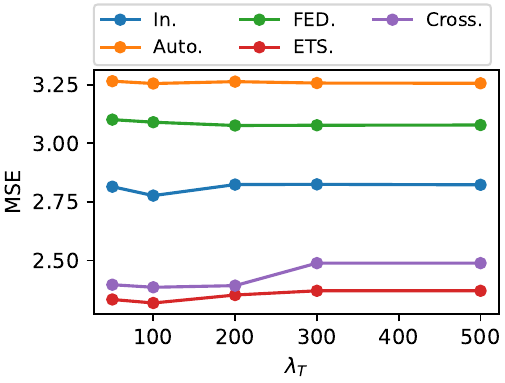}
    }%
    \subfigure[Illness-$\lambda_P$]{
        \includegraphics[width=0.24\textwidth]{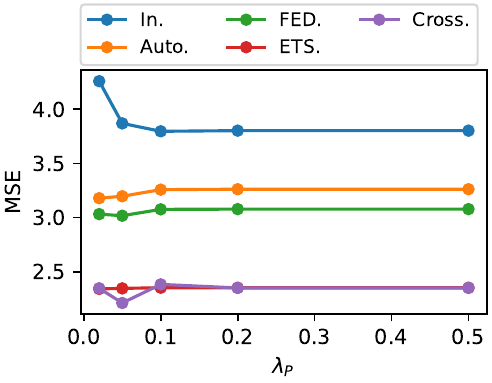}
    }%
    \subfigure[Illness-$\lambda_N$]{
        \includegraphics[width=0.24\textwidth]{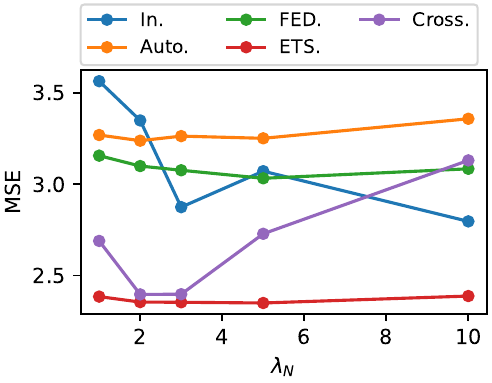}
    }%
    \subfigure[Illness-$lr$]{
        \includegraphics[width=0.24\textwidth]{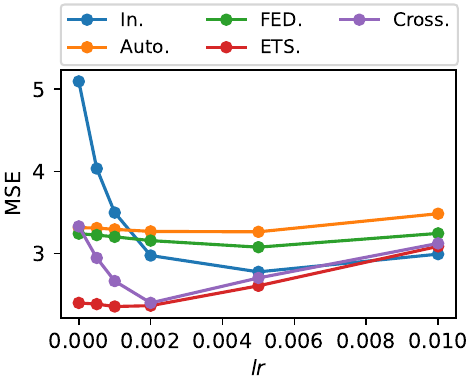}
    }
    \subfigure[Traffic-$\lambda_T$]{
        \includegraphics[width=0.24\textwidth]{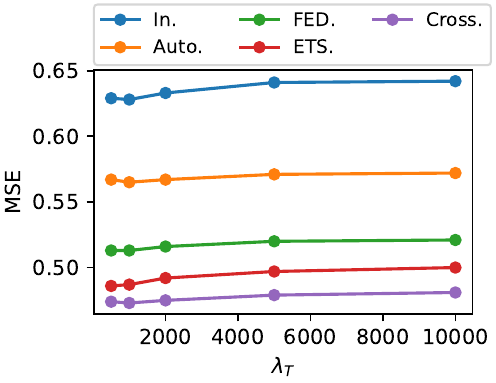}
    }%
    \subfigure[Traffic-$\lambda_P$]{
        \includegraphics[width=0.24\textwidth]{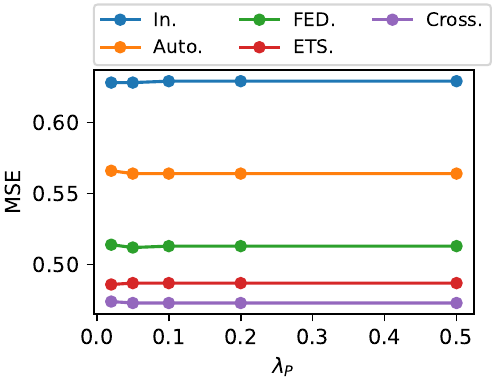}
    }%
    \subfigure[Traffic-$\lambda_N$]{
        \includegraphics[width=0.24\textwidth]{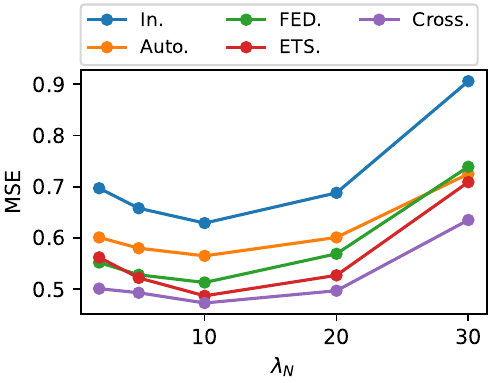}
    }%
    \subfigure[Traffic-$lr$]{
        \includegraphics[width=0.24\textwidth]{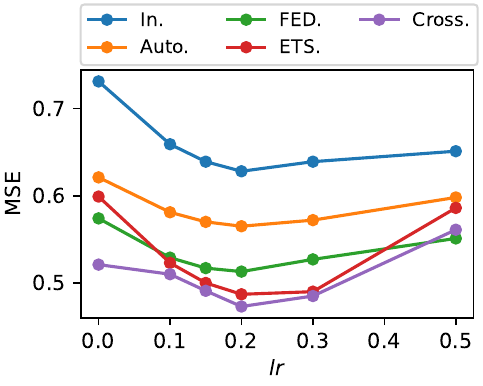}
    }
    \label{fig:parameter_sensitivity}
\end{figure*}

%% file: table/case_study_short.tex
\begin{figure*}[t]
    \centering  
    \caption{Case study and visualization of our proposed methods on various cases across different datasets and models, where \textcolor[rgb]{0, 0.439, 0.753}{BLUE} lines represent ground-truth, \textcolor{orange}{ORANGE} lines represent original forecasting results, and \textcolor[rgb]{0, 0.5, 0}{GREEN} lines represent the forecasts after employing our approach. Besides, $ts120$ denotes this case is sampled at timestep-120 in the test set, and so on.}
    \subfigure[Illness-ts120]{
        \includegraphics[width=0.24\textwidth]{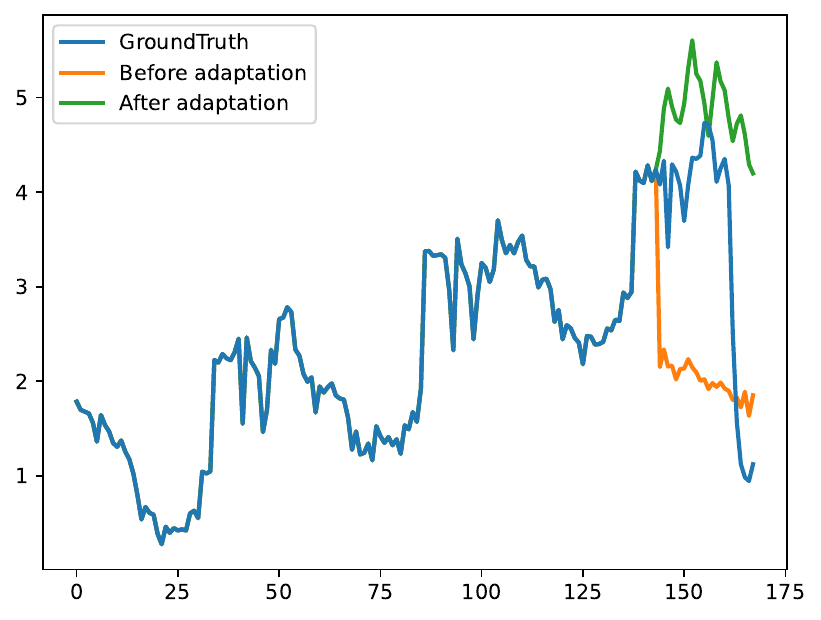}
    }%
    \subfigure[Illness-ts160]{
        \includegraphics[width=0.24\textwidth]{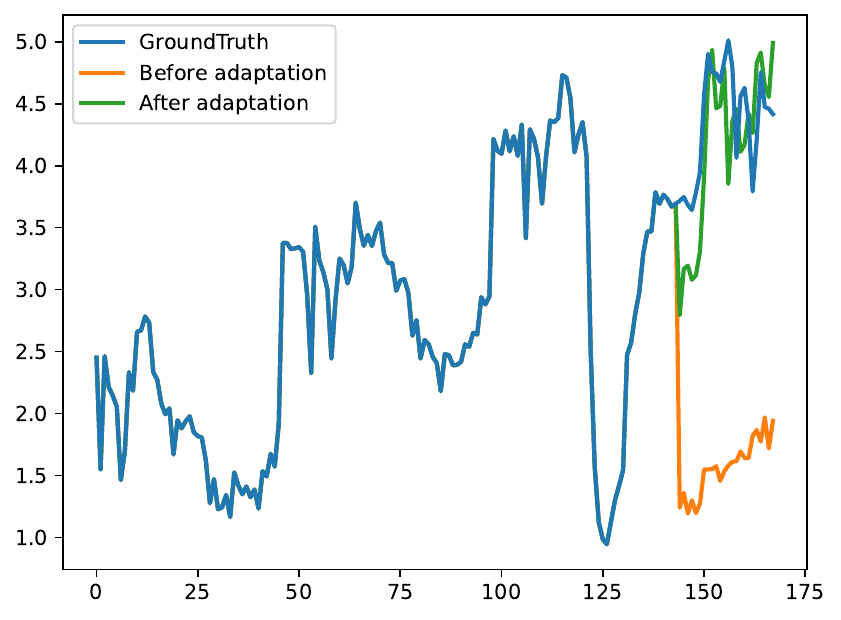}
    }%
    \subfigure[ETTh1-ts250]{
        \includegraphics[width=0.24\textwidth]{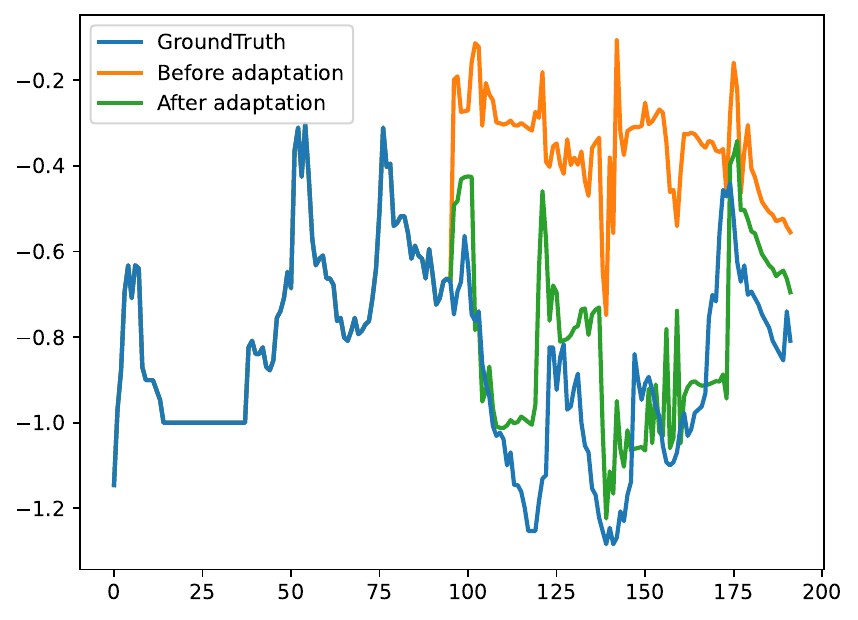}
    }%
    \subfigure[ETTh1-ts650]{
        \includegraphics[width=0.24\textwidth]{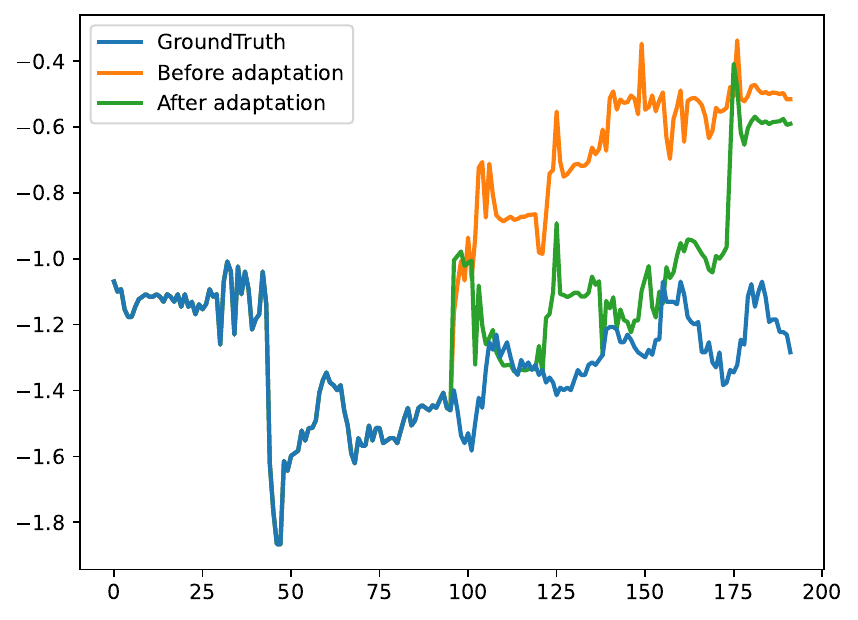}
    }
    \label{fig:visualization-of-prediction}
\end{figure*}

%% file: sec-appendix-full.tex
\appendix

\numberwithin{equation}{section}
\section*{Appendix}

\section{Proofs for the theoretical results}\label{sec:appendix_proof}

\ifkdd

Proofs for Theorem \ref{thm:error-global} and Theorem \ref{thm:error-context} can be found in our full version paper\footnote{\url{https://arxiv.org/abs/2310.14838}}. 

\else

In this section, we give the proofs for the expected risks in Theorem \ref{thm:error-global} and Theorem \ref{thm:error-context}. To start with, we give the following lemma to perform the usual bias/variance decomposition. The risk notations $\mathcal{R}$ and $\mathcal{R}^*$ are defined in \sref{sec:theoretical_analysis}.

\begin{lemma}[Risk decomposition]
    \label{lma:risk_decomposition}
    For a set of random variables $\alpha_1, \cdots, \alpha_K$ used as parameters for each group, we have:
    \begin{align*}
        \mathcal{R}(\alpha_1, \cdots, \alpha_K) - \mathcal{R^*} = 
        \underbrace{
            \sum_{i=1}^K \left\| 
                \mathbb E[\alpha_i] - \theta_i
            \right\|^2_{\psi_i}
        }_{\text{bias part}}
        +\underbrace{
            \sum_{i=1}^K \mathbb E\left[\left\|
                \alpha_i -\mathbb E[\alpha_i]
            \right\|^2_{\psi_i}
        \right]}_{\text{variance part}},
    \end{align*}
    where $\psi_i = X_i^\top X_i$, $\overline \theta = (\sum_{i=1}^K{\psi_i})^{-1} (\sum_{i=1}^K{\psi_i \theta_i})$, and $ \|\theta\|^2_{\psi_i} = \theta^\top \psi_i \theta$.
\end{lemma}
\begin{proof}
    This is a direct corollary of \cite[Proposition 3.3]{bach2023learning}.
\end{proof}

Based on the risk decomposition, we give the computation for GLR and CLR respectively, as follows.
\allowdisplaybreaks

\subsection{Proof for Theorem \ref{thm:error-global}} \label{sec:appendix_thm_global}

Set $\alpha_1=\cdots=\alpha_K=\hat\theta$ in Lemma \ref{lma:risk_decomposition}, we obtain:
\begin{align*}
    \mathcal{R}\underbrace{\left(
        \hat\theta,\cdots,\hat\theta
    \right)}_{K} - \mathcal{R^*} = 
    \underbrace{
        \sum_{i=1}^K \left\| 
            \mathbb E[\hat\theta] - \theta_i
        \right\|^2_{\psi_i}
    }_{\text{bias part}}
    +\underbrace{
        \sum_{i=1}^K \mathbb E\left[\left\|
            \alpha_i -\mathbb E[\hat\theta]
        \right\|^2_{\psi_i}
    \right]}_{\text{variance part}}.
\end{align*}

We first compute the expectation in the bias part. Recall that,
\begin{align*}
    \hat \theta=\arg\min_{\theta}\sum_{i=1}^K\|Y_i-X_i\theta\|_2^2,
\end{align*}
which has a closed-form solution, as follows:
\begin{align}
    \label{eq:thm1_hat_theta}
    \hat \theta=\left(
    	\sum_{i=1}^K \psi_i
    \right)^{-1} \left(
    	\sum_{i=1}^K X_i^\top Y_i
    \right)
    =\left(
    	\sum_{i=1}^K \psi_i
    \right)^{-1} \left(
    	\sum_{i=1}^K X_i^\top (X_i\theta_i+\epsilon_i)
    \right).
\end{align}

Therefore, we can obtain the expectation:
\begin{align}
    \nonumber
    \mathbb E[\hat \theta]&=\left(
    	\sum_{i=1}^K \psi_i
    \right)^{-1} \left(
    	\sum_{i=1}^K X_i^\top (X_i\theta_i+\mathbb E[\epsilon_i])
    \right)\\
    \label{eq:thm1_e_hat_theta}
        &=\left(
    	\sum_{i=1}^K \psi_i
    \right)^{-1} \left(
    	\sum_{i=1}^K \psi_i\theta_i
    \right):=\overline \theta,
\end{align}
which implies the bias part. Then we compute the variance part $V$:
\begin{align*}
    V&=\sum_{i=1}^K \mathbb E\left[
    	\left\|\hat \theta-\mathbb E[\hat \theta]\right\|^2_{\psi_i}
    \right]  \\
    &=\sum_{i=1}^K \mathbb  E\left[
    	\left\|
    		\left(
                \sum_{j=1}^K \psi_j
            \right)^{-1} \left(
                \sum_{j=1}^K X_j^\top \epsilon_j
            \right)
    	\right\|^2_{\psi_i}
    \right] \cdots \text{by \Eqref{eq:thm1_hat_theta} and \Eqref{eq:thm1_e_hat_theta}}\\
    &=\sum_{i=1}^K \mathbb  E\left[
    	\left(
        	\sum_{j=1}^K X_j^\top \epsilon_j
        \right)^\top
        \underbrace{\left(
        	\sum_{j=1}^K \psi_j
        \right)^{-1}
    	\psi_i
        \left(
        	\sum_{j=1}^K \psi_j
        \right)^{-1}}_{\text{denoted by }P_i} 
        \left(
        	\sum_{j=1}^K X_j^\top \epsilon_j
        \right)
    \right] \\
    &=\sum_{i=1}^K \mathbb  E\left[
    	\left(
        	\sum_{j=1}^K X_j^\top \epsilon_j
        \right)^\top
        P_i
        \left(
        	\sum_{j=1}^K X_j^\top \epsilon_j
        \right)
    \right]  \\
    &=\sum_{i=1}^K \mathbb  E\left[
    	\left(
        	\sum_{j=1}^K  \epsilon_j^\top X_j
        \right)
        P_i
        \left(
        	\sum_{j=1}^K X_j^\top \epsilon_j
        \right)
    \right]  \\
    &=\sum_{i=1}^K \mathbb  E\left[
    	\sum_{j=1}^K\sum_{k=1}^K
    	\epsilon_j^\top X_j
        P_i
    	X_k^\top \epsilon_k
    \right]  \\
    &=\sum_{i=1}^K \sum_{j=1}^K \sum_{k=1}^K \mathbb E\left[
    \epsilon_j^\top X_j
    P_i
    X_k^\top \epsilon_k
    \right]  \\
    &=\sum_{i=1}^K \sum_{j=1}^K\mathbb  E\left[
	\epsilon_j^\top X_j P_i X_j^\top \epsilon_j
    \right] \cdots \text{by }\mathbb E[\epsilon_j^\top X_j P_i X_k^\top \epsilon_k]=0\text{ if }j\neq k\\
    &=\sum_{i=1}^K \sum_{j=1}^K \mathbb E\left[
    	\text{tr}(X_j P_i X_j^\top \epsilon_j \epsilon_j^\top)
    \right]   \cdots \text{by tr}(AB)=\text{tr}(BA)  \\
    &=\sigma^2 \sum_{i=1}^K \sum_{j=1}^K
    	\text{tr}(X_j P_i X_j^\top)
    \cdots \text{by }\mathbb E[\epsilon_j \epsilon_j^\top]=\sigma^2 I_n\\
    &=\sigma^2 \sum_{i=1}^K \sum_{j=1}^K \text{tr}\left(
    	X_j
    	\left(
        	\sum_{k=1}^K \psi_k
        \right)^{-1}
    	\psi_i
        \left(
        	\sum_{k=1}^K \psi_k
        \right)^{-1}
    	X_j^\top
    \right)\cdots \text{by plugging }P_i\\
    &=\sigma^2 \sum_{i=1}^K \sum_{j=1}^K \text{tr}\left(
    	\psi_j
    	\left(
        	\sum_{k=1}^K \psi_k
        \right)^{-1}
    	\psi_i
        \left(
        	\sum_{k=1}^K \psi_k
        \right)^{-1}
    \right)\cdots \text{by }\psi_j=X_j^\top X_j \\
    &=\sigma^2 \text{tr}\left(\sum_{i=1}^K \sum_{j=1}^K 
    \psi_j
    \left(
        \sum_{k=1}^K \psi_k
    \right)^{-1}
    \psi_i
    \left(
        \sum_{k=1}^K \psi_k
    \right)^{-1}
    \right) \\
    &=\sigma^2 \text{tr}\left(
    	\left(
        	\sum_{j=1}^K \psi_j
        \right)
    	\left(
        	\sum_{k=1}^K \psi_k
        \right)^{-1}
    	\left(
        	\sum_{i=1}^K \psi_i
        \right)
        \left(
        	\sum_{k=1}^K \psi_k
        \right)^{-1}
    \right) \\
    &=\sigma^2 \text{tr}(I\cdot I)  \\
    &= \sigma^2d.  
\end{align*}

\subsection{Proof for Theorem \ref{thm:error-context}} \label{sec:appendix_thm_context}

Set $\alpha_i=\hat\theta_i$ in Lemma \ref{lma:risk_decomposition} for $i\in[K]$, we obtain:
\begin{align}
    \label{eq:appendix_thm2_risk}
    \mathcal{R}\left(
        \hat\theta_1,\cdots,\hat\theta_K
    \right) - \mathcal{R^*} = 
    \underbrace{\sum_{i=1}^K 
        \left\| 
            \mathbb E[\hat\theta_i] - \theta_i
        \right\|^2_{\psi_i}
    }_{\text{bias part}}
    +  \underbrace{\sum_{i=1}^K
        \mathbb E\left[\left\|
            \hat\theta_i -\mathbb E[\hat\theta_i]
        \right\|^2_{\psi_i}
    \right]}_{\text{variance part}}.
\end{align}

From Proposition 3.4 and Proposition 3.5 in \cite{bach2023learning}, for each $i\in[K]$:
\begin{align}
    \label{eq:appendix_thm2_ref}
    \left\| 
        \mathbb E[\hat\theta_i] - \theta_i
    \right\|^2_{\psi_i}
     =0,\quad 
        \mathbb E\left[\left\|
            \hat\theta_i -\mathbb E[\hat\theta_i]
        \right\|^2_{\psi_i}
    \right] = \sigma^2 d.
\end{align}

Combining \Eqref{eq:appendix_thm2_risk} and \Eqref{eq:appendix_thm2_ref}, we obtain the desired result.

\fi

\ifkdd
\else
\input{table/dataset_statistics}
\fi

\input{table/param_range}
\section{Details of experiments} 

\subsection{Datasets details} \label{dataset_detailed}

\ifkdd

We conduct our experiment on 8 popular datasets following previous researches \cite{Informer, Autoformer}, which are publicly available at \url{https://github.com/zhouhaoyi/Informer2020} and \url{https://github.com/thuml/Autoformer}.

\else

We conduct our experiment on 8 popular datasets following previous researches \cite{Informer, Autoformer}. The statistics of these datasets are summarized in Table \ref{datasets}, and publicly available at \url{https://github.com/zhouhaoyi/Informer2020} and \url{https://github.com/thuml/Autoformer}.

\fi

\par \textbf{(1) ETTh1/ETTh2/ETTm1/ETTm2}. ETT dataset contains 7 indicators collected from electricity transformers from July 2016 to July 2018, including useful load, oil temperature, etc. Data points are recorded hourly for ETTh1 and ETTh2, while recorded every 15 minutes for ETTm1 and ETTm2.
\par \textbf{(2) Electricity}. Electricity dataset contains the hourly electricity consumption (in KWh) of 321 customers from 2012 to 2014.
\par \textbf{(3) Traffic}. Traffic dataset contains hourly road occupancy rate data measured by different sensors on San Francisco Bay area freeways in 2 years. The data is from the California Department of Transportation.
\par \textbf{(4) Illness}. Illness dataset includes 7 weekly recorded indicators of influenza-like illness patients data from Centers for
Disease Control and Prevention of the United States between 2002 and 2021.
\par \textbf{(5) Weather}. Weather dataset contains 21 meteorological indicators, like temperature, humidity, etc. The dataset is recorded every 10 minutes for the 2020 whole year.

\subsection{Baseline models} \label{baseline_models}
\par We briefly describe our selected baseline models:
\par \textbf{(1) Informer} \cite{Informer} utilizes ProbSparse self-attention and distillation operation to capture cross-time dependency.
\par \textbf{(2) Autoformer} \cite{Autoformer} utilizes series decomposition block architecture with Auto-Correlation to capture cross-time dependency.
\par \textbf{(3) FEDformer} \cite{FEDformer} presents a sparse
representation within the frequency domain and proposes frequency-enhanced blocks to capture the cross-time dependency.
\par \textbf{(4) ETSformer} \cite{ETSformer} exploits the principle of exponential smoothing and leverages exponential smoothing attention and frequency attention.
\par \textbf{(5) Crossformer} \cite{Crossformer} utilizes a dimension-segment-wise embedding and introduces a two-stage attention layer to capture the cross-time and cross-dimension dependency.
\par \textbf{(6) DLinear} \cite{LTSF-Linear} proposes a linear forecasting model, along with seasonal-trend decomposition or normalization operations.
\par \textbf{(7) PatchTST} \cite{PatchTST} utilizes Transformer encoders as model backbone, and proposes patching and channel independence techniques for better prediction.

\subsection{Hyper-parameters} \label{hyper-parameters}
\par For all experiments, during the training process, we use the same hyper-parameters as reported in the corresponding papers \cite{Informer, Autoformer, FEDformer, ETSformer, Crossformer, LTSF-Linear}, \eg, encoder/decoder layers, model hidden dimensions, head numbers of multi-head attention and batch size. 
\par As for hyper-parameters for adaptation process, there exists 4 major hyper-parameters for \method{SOLID}, including $\lambda_T$, $\lambda_P$, $\lambda_N$, $lr$ (We report the ratio $\nicefrac{lr}{lr_{\text{train}}}$ between adaptation learning rate $lr$ and the training learning rate $lr_{\text{train}}$ as an alternative).
We select the setting which performs the best on the validation set. The search range for the parameters is presented in Table \ref{tbl:grid-search}.

\subsection{Specific structures of $g_\Phi$ and $h_\theta$} \label{specific_structure_of_different_models}
\par We divide our baseline TSF models into three categories, including Encode-decoder-based, Encoder-based, and Linear-based architectures for explanation.

\begin{itemize}[leftmargin=*]
    \item Encode-decoder-based (\eg, Informer, Autoformer, FEDformer, ETSformer, and Crossformer): $h_\theta$ refers to the linear projection layer at the top of decoders, and the other parts of decoders and the entire encoders correspond to $g_\Phi$.
    \item Encoder-based (\eg, PatchTST): $h_\theta$ refers to the top linear prediction layer at the top of the encoders, and the other parts (including self-attention and embedding layers) correspond to $g_\Phi$.
    \item Linear-based (\eg, DLinear): Since Linear models only include linear layers for modeling, $h_\theta$ refers to the linear layers, and there is no $g_\Phi$.
\end{itemize}

\section{Further Experiment Results}

\subsection{Full results}\label{sec:appendix_full_result}

In this section, we report the full results of \sref{sec:main_result} and \sref{sec:further_analysis}. 

\begin{itemize}[leftmargin=*]
    \item Table \ref{tab:Ablation_study} details the ablation study's results across five models and three datasets.  
    \item Tables \ref{tbl:RevIN} and \ref{tbl:Dish-TS} compare the effectiveness of our approach against RevIN and Dish-TS in addressing distribution shifts.
    \item Table \ref{tbl:adapt_whole_model} compares the results of two tuning strategies, \ie, full parameter fine-tuning vs. prediction layer-only fine-tuning, applied to five models and three datasets. The prediction lengths are 24 (for Illness) and 96 (for Traffic and ETTh1).
    \item Table \ref{tbl:Main_Results_Part2} details the benchmark results for ETTm1, ETTm2, and Weather datasets, where our proposed \method{Reconditioner} identified relatively lower metrics.
\end{itemize}

\input{table/ablation_study}

\input{table/revin}

\input{table/dishts}
\input{table/tuning_strategies_1}

\input{table/extended_main_results}
\begin{figure}[t]
    \centering
    \subfigure[context: temporal segments]{
        \label{fig:intro-confounder-time}
        \includegraphics[width=0.22\textwidth,trim={10 10 10 10},clip]{figure/period24_timesteps_pca_random.png}
    }
    \subfigure[context: periodic phases]{
        \label{fig:intro-confounder-phases}
        \includegraphics[width=0.22\textwidth,trim={10 10 10 10},clip]{figure/period24_6clusters_pca.png}
    }
    \caption{Illustration for two contexts on ETTh1 training dataset, which has a period of 24 and length of 8760. Different colors represent different values of contexts.}
    \label{fig:intro-confounder}
\end{figure}

\subsection{Visualizing the impact of contexts on data distribution}\label{sec:appendix_context_visual}

To demonstrate the influence of context on data distribution, we utilized the Autoformer \cite{Autoformer} to extract latent representations from the ETTh1 dataset and applied PCA for dimensionality reduction and visualization. This ensures the spatial positions of the data points in the figure represent their original distribution. 

We marked two observed contexts -- temporal segments and periodic phases -- on the figures. Unobserved contexts, being difficult to visualize, are not displayed. Figure \ref{fig:intro-confounder-time} reveals a progressive outward shift in data distribution with increasing temporal segments. Similarly, Figure \ref{fig:intro-confounder-phases} shows that changes in the periodic phase lead to a rotational shift in data distribution. Note that if CDS doesn't exist, different colors (denoting contexts) should be scattered and randomly mixed in the figure since colors are independent of spatial positions, which is not the case shown in Figure \ref{fig:intro-confounder}. Therefore, it is evident that these contexts markedly affect data distribution.

\balance

\subsection{Correlation of $\delta_T$ and MAE improvement}\label{sec:app_corr_t}

\begin{figure}[t]
    \centering
    \includegraphics[width=0.47\textwidth,trim={10 10 10 10},clip]{figure/exp_correlation_time.pdf}
    \caption{The relationship $\bm{\log_{10} \delta_T}$ (X-axis) and MAE improvements achieved by \method{SOLID} (Y-axis) for 8 datasets and 6 models. }
    \label{fig:indicator-time}
\end{figure}

In \sref{sec:relation_reconditionor_solid}, we found that the correlation between $\delta_P$ and MAE improvement is very strong. We display the correlation between $\delta_T$ and MAE improvement in Figure \ref{fig:indicator-time} as well. Unlike $\delta_P$, the relation between $\delta_T$ and MAE improvement is not as straightforward. Despite showing an increasing trend in Figure \ref{fig:indicator-time}, the Spearman correlation coefficient is merely 0.2129. This implies that, while there is evidence of CDS stemming from temporal segments as detected by \method{Reconditionor}, \method{SOLID} is comparatively less effective at mitigating it when compared to CDS caused by periodic phases. One possible explanation is that data generated within the same phase tends to follow a more predictable pattern, while data within the same temporal segment exhibits greater diversity and uncertainty, which may limit the utility of selecting data from the same segment to address CDS caused by temporal segments. We leave further investigation of it to future work.

\subsection{Parameter sensitivity analysis} \label{sec:appendix-parameter-sensitivity}

\par Within our proposed approach, several crucial parameters are presented, including $\lambda_T$, which controls the time range of preceding data for selection; $\lambda_P$, which governs the acceptable threshold for periodic phase difference; $\lambda_N$, which determines the number of similar samples to be selected for model adaptation; And $lr$, which regulates the extent of adaptation on the prediction layer for the models. The search range for these parameters is presented in Table \ref{tbl:grid-search}. The results of parameter sensitivity analysis are visually presented in Figure \ref{fig:parameter_sensitivity}.

\par Firstly, we discover that our proposed \method{SOLID} is insensitive to $\lambda_T$ and $\lambda_P$ parameters, based on the results obtained.
Regarding the $\lambda_N$ parameter, the selection of insufficient samples would increase the variance during adaptation due to the data shortage. Conversely, selecting excessive samples carries the risk of including samples with unobserved irrelevant contexts, thereby deteriorating the model’s performance.
Lastly, For parameter $lr$, a very small $lr$ leads to inadequate model adaptation, preventing the model from effectively addressing CDS and resulting in a bias towards the test sample (Theorem \ref{thm:error-global}). Conversely, a too large value for $lr$ can lead to excessive adaptation, which also risks bringing in substantial variance to the model (Theorem \ref{thm:error-context}). Therefore, a well-selected learning rate will contribute to an optimal trade-off between bias and variance.

\subsection{Case study} \label{sec:appendix-visulization}

\par In addition, we conduct a case study and visualize our proposed method on various cases across different datasets and models, as presented in Figure \ref{fig:visualization-of-prediction}.
\par Specifically, we plot the figures on these combinations: Illness (timestep 120~\&~160, variate 7) on Crossformer, Traffic (timestep 200~\&~550, variate 862) on FEDformer, Electricity (timestep 1500~\&~3000, variate 321) on Autoformer, and ETTh1 (timestep 250~\&~650, variate 7) on Informer, corresponding to Figure \ref{fig:visualization-of-prediction} (a)-(h).
The visualization vividly illustrates the effectiveness of our approach in improving forecasting performance.

\input{table/param_experiment}
\input{table/case_study}

%% file: table/dataset_statistics.tex
\begin{table*}[t]
  \centering
  \caption{The statistics of 8 datasets, including the number of variates, total timesteps, and the frequency of data sampling.}
    \begin{tabular}{ccccccccc}
    \toprule
    \textbf{Dataset} & \textbf{ETTh1} & \textbf{ETTh2} & \textbf{ETTm1} & \textbf{ETTm2} & \textbf{Electricity} & \textbf{Traffic} & \textbf{Weather} & \textbf{Illness} \\
    \midrule
    \midrule
    \textbf{Variates} & 7     & 7     & 7     & 7     & 321   & 862   & 21    & 7 \\
    \textbf{Timesteps} & 17,420 & 17,420 & 69,680 & 69,680 & 26,304 & 17,544 & 52,695 & 966 \\
    \textbf{Frequency} & 1hour & 1hour & 15min & 15min & 1hour & 1hour & 10min & 1week \\
    \bottomrule
    \end{tabular}%
  \label{datasets}%
\end{table*}%

%% file: table/ablation_study.tex
\begin{table}[t]
  \centering
  \caption{Extended table of Figures \ref{fig:extra-ablation-illness}-\ref{fig:extra-ablation-traffic}: Ablation study on \method{SOLID}. }
  \resizebox{\linewidth}{!}{  
    \begin{tabular}{c|c|cc|cc|cc}
    \toprule
    \multicolumn{2}{c|}{\textbf{Dataset}}     & \multicolumn{2}{c|}{\textbf{Electricity}} & \multicolumn{2}{c|}{\textbf{Traffic}} & \multicolumn{2}{c}{\textbf{Illness}} \\
    \midrule
    \multicolumn{2}{c|}{\textbf{Metric}} & \textbf{MSE} & \textbf{MAE} & \textbf{MSE} & \textbf{MAE} & \textbf{MSE} & \textbf{MAE} \\
    \midrule
    \midrule
    \rotateablation{\textbf{Informer}} & \textbf{/} & 0.321  & 0.407  & 0.731  & 0.406  & 5.106  & 1.534  \\
          & \textbf{T} & 0.252  & 0.362  & 0.636  & 0.396  & 4.443  & 1.427  \\
          & \textbf{T+P} & 0.248  & 0.358  & \textbf{0.628 } & \textbf{0.393 } & 2.901  & 1.186  \\
          & \textbf{T+P+S} & \textbf{0.245 } & \textbf{0.355 } & \textbf{0.628 } & \textbf{0.393 } & \textbf{2.874 } & \textbf{1.150 } \\
    \midrule
    \rotateablation{\textbf{Autoformer}} & \textbf{/} & 0.207  & 0.324  & 0.620  & 0.391  & 3.314  & 1.245  \\
          & \textbf{T} & 0.193  & 0.309  & 0.579  & 0.380  & 3.333  & 1.251  \\
          & \textbf{T+P} & 0.196  & 0.309  & 0.577  & 0.380  & 2.981  & 1.201  \\
          & \textbf{T+P+S} & \textbf{0.189 } & \textbf{0.304 } & \textbf{0.565 } & \textbf{0.377 } & \textbf{2.737 } & \textbf{1.118 } \\
    \midrule
    \rotateablation{\textbf{FEDformer}} & \textbf{/} & 0.188  & 0.304  & 0.574  & 0.356  & 3.241  & 1.252  \\
          & \textbf{T} & 0.173  & 0.284  & 0.522  & 0.346  & 3.234  & 1.253  \\
          & \textbf{T+P} & \textbf{0.172 }  & \textbf{0.283 } & 0.520  & 0.345  & 2.953  & 1.189  \\
          & \textbf{T+P+S} & \textbf{0.172 } & 0.284  & \textbf{0.513 } & \textbf{0.344 } & \textbf{2.707 } & \textbf{1.123 } \\
    \midrule
    \rotateablation{\textbf{ETSformer}} & \textbf{/} & 0.187  & 0.304  & 0.599  & 0.386  & 2.397  & 0.993  \\
          & \textbf{T} & 0.174  & 0.289  & 0.493  & 0.356  & 2.377  & 0.989  \\
          & \textbf{T+P} & 0.173  & 0.287  & 0.489  & 0.356  & 2.334  & 0.977 \\
          & \textbf{T+P+S} & \textbf{0.171 } & \textbf{0.285 } & \textbf{0.487 } & \textbf{0.354 } & \textbf{2.262 } & \textbf{0.955 } \\
    \midrule
    \rotateablation{\textbf{Crossformer}} & \textbf{/} & 0.184  & 0.297  & 0.521  & 0.297  & 3.329  & 1.275  \\
          & \textbf{T} & 0.188  & 0.301  & 0.482  & 0.273  & 2.705  & 1.098  \\
          & \textbf{T+P} & 0.184  & 0.297  & 0.476  & 0.270  & 2.469  & 0.993  \\
          & \textbf{T+P+S} & \textbf{0.182 } & \textbf{0.295 } & \textbf{0.473 } & \textbf{0.267 } & \textbf{2.353 } & \textbf{0.986 } \\
    \bottomrule
    \end{tabular}%
    } 
  \label{tab:Ablation_study}%
  \vspace{-0.3cm}
\end{table}%

%% file: table/revin.tex
\begin{table}[t]
  \centering
  \caption{Extended table of Figure \ref{fig:extra-revin}: Comparison experiments between our proposed \method{SOLID} and RevIN for tackling distribution shift. The TSF model is Autoformer.}
  
  \resizebox{\linewidth}{!}{
    \begin{tabular}{c|c|cc|cc|cc}
    \toprule
    \multicolumn{2}{c|}{\textbf{Method}} & \multicolumn{2}{c|}{\textbf{+\method{SOLID}}} & \multicolumn{2}{c|}{\textbf{+RevIN}} & \multicolumn{2}{c}{\textbf{+RevIN +\method{SOLID}}} \\
    \midrule
    
    \multicolumn{2}{c|}{\textbf{Metric}} & \textbf{MSE}   & \textbf{MAE}   & \textbf{MSE}   & \textbf{MAE}   & \textbf{MSE}   & \textbf{MAE} \\
    \midrule
    \midrule
    ETTh1 & 96    & \textbf{0.430 } & \textbf{0.442 } & 0.465  & 0.453  & 0.462  & 0.450 \\
    ETTh2 & 96    & \textbf{0.362 } & 0.404  & 0.396  & 0.401  & 0.393  & \textbf{0.398 } \\
    ETTm1 & 96    & \textbf{0.443 } & \textbf{0.448 } & 0.555  & 0.479  & 0.553  & 0.477  \\
    ETTm2 & 96    & 0.312  & 0.345  & 0.232 & 0.307 & \textbf{0.230 } & \textbf{0.305 } \\
    Electricity & 96    & 0.189  & 0.304  & 0.182 & 0.289 & \textbf{0.176 } & \textbf{0.284 } \\
    Traffic & 96    & 0.565  & 0.376  & 0.637 & 0.392 & \textbf{0.552 } & \textbf{0.357 } \\
    Weather & 96    & 0.259  & 0.334  & 0.209 & 0.254 & \textbf{0.208 } & \textbf{0.253 } \\
    Illness & 24    & \textbf{2.737 } & 1.118  & 2.810  & 1.119 & 2.795  & \textbf{1.107 }\\
    \bottomrule
    \end{tabular}%
    }
  \label{tbl:RevIN}%
\end{table}%

%% file: table/dishts.tex
\begin{table}[t]
  \centering
  \caption{Extended table of Figure \ref{fig:extra-dishts}: Comparison experiments between our proposed \method{SOLID} and Dish-TS for tackling distribution shift. The TSF model is Autoformer.}
  
  \resizebox{\linewidth}{!}{
    \begin{tabular}{c|c|cc|cc|cc}
    \toprule
    \multicolumn{2}{c|}{\textbf{Method}} & \multicolumn{2}{c|}{\textbf{+\method{SOLID}}} & \multicolumn{2}{c|}{\textbf{+Dish-TS}} & \multicolumn{2}{c}{\textbf{+Dish-TS +\method{SOLID}}}\\
    \midrule
    \multicolumn{2}{c|}{\textbf{Metric}} & \textbf{MSE} & \textbf{MAE} & \textbf{MSE} & \textbf{MAE} & \textbf{MSE} & \textbf{MAE}\\
    \midrule
    \midrule
    ETTh1 & 96    & 0.430  & 0.442  & 0.419  & 0.430  & \textbf{0.416 } & \textbf{0.428 }\\
    ETTh2 & 96    & \textbf{0.362 } & 0.404  & 0.369  & 0.394  & 0.366  & \textbf{0.393 } \\
    ETTm1 & 96    & \textbf{0.443 } & 0.448  & 0.488  & 0.452  & 0.486  & \textbf{0.449 } \\
    ETTm2 & 96    & 0.312  & 0.345  & 0.213  & 0.295  & \textbf{0.209 } & \textbf{0.291 } \\
    Electricity & 96    & 0.189  & 0.304  & 0.169  & 0.279  & \textbf{0.162 } & \textbf{0.273 } \\
    Traffic & 96    & 0.565  & 0.376  & 0.586  & 0.381  & \textbf{0.518 } & \textbf{0.345 } \\
    Weather & 96    & 0.259  & 0.334  & 0.194  & 0.244  & \textbf{0.192 } & \textbf{0.243 } \\
    Illness & 24    & 2.737  & 1.118  & 2.587  & 1.079  & \textbf{2.571 } & \textbf{1.063 }\\
    \bottomrule
    \end{tabular}%
    }
  \label{tbl:Dish-TS}%
\end{table}%

%% file: table/tuning_strategies_1.tex
\begin{table}[t]
  \centering
  \caption{Extended table of Figure \ref{fig:extra-tuning}: Performance comparison between adaptation on prediction layer (PL) versus adaptation on entire model (EM).}
  \small
  \resizebox{\linewidth}{!}{  
\begin{tabular}{c|c|cc|cc|cc}
\toprule
\multicolumn{2}{c|}{\textbf{Dataset}}       & \multicolumn{2}{c|}{\textbf{ETTh1}} & \multicolumn{2}{c|}{\textbf{Traffic}} & \multicolumn{2}{c}{\textbf{Illness}} \\
\midrule
\multicolumn{2}{c|}{\textbf{Metric}} & \textbf{MSE} & \textbf{MAE} & \textbf{MSE} & \textbf{MAE} & \textbf{MSE} & \textbf{MAE} \\
\midrule
\midrule
\textbf{Informer} & \textbf{/} & 0.948  & 0.774  & 0.731  & 0.406  & 5.096  & 1.533  \\
          & \textbf{+\method{SOLID-PL}} & \textbf{0.684 } & \textbf{0.586 } & \textbf{0.628 } & \textbf{0.393 } & \textbf{2.874 } & \textbf{1.150 } \\
          & \textbf{+\method{SOLID-EM}} & 0.817  & 0.656  & 0.692  & 0.403  & 2.991  & 1.208  \\
\midrule
\textbf{Autoformer} & \textbf{/} & 0.440  & 0.444  & 0.621  & 0.391  & 3.314  & 1.245  \\
          & \textbf{+\method{SOLID-PL}} & \textbf{0.430 } & \textbf{0.442 } & \textbf{0.565 } & \textbf{0.376 } & \textbf{2.737 } & \textbf{1.118 } \\
          & \textbf{+\method{SOLID-EM}} & 0.436  & 0.443  & 0.597  & 0.385  & 2.891  & 1.150  \\
\midrule
\textbf{FEDformer} & \textbf{/} & 0.375  & 0.414  & 0.574  & 0.356  & 3.241  & 1.252  \\
          & \textbf{+\method{SOLID-PL}} & \textbf{0.370 } & \textbf{0.410 } & \textbf{0.513 } & \textbf{0.344 } & \textbf{2.707 } & \textbf{1.123 } \\
          & \textbf{+\method{SOLID-EM}} & 0.373  & 0.413  & 0.539  & 0.348  & 2.874  & 1.201  \\
\midrule
\textbf{ETSformer} & \textbf{/} & 0.495  & 0.480  & 0.599  & 0.386  & 2.397  & 0.993  \\
          & \textbf{+\method{SOLID-PL}} & \textbf{0.491 } & \textbf{0.478 } & \textbf{0.487 } & \textbf{0.354 } & \textbf{2.262 } & \textbf{0.955 } \\
          & \textbf{+\method{SOLID-EM}} & 0.492  & 0.479  & 0.517  & 0.368  & 2.353  & 0.986  \\
\midrule
\textbf{Crossformer} & \textbf{/} & 0.411  & 0.432  & 0.521  & 0.297  & 3.329  & 1.275  \\
          & \textbf{+\method{SOLID-PL}} & \textbf{0.382 } & \textbf{0.415 } & \textbf{0.473 } & \textbf{0.267 } & \textbf{2.397 } & \textbf{1.052 } \\
          & \textbf{+\method{SOLID-EM}} & 0.401  & 0.427  & 0.493  & 0.273  & 2.653  & 1.074  \\
\bottomrule
\end{tabular}%
}

    \label{tbl:adapt_whole_model}
\end{table}%

%% file: table/case_study.tex
\begin{figure*}[t]
    \centering  
    \caption{Case study and visualization of our proposed methods on various cases across different datasets and models, where \textcolor[rgb]{0, 0.439, 0.753}{BLUE} lines represent ground-truth, \textcolor{orange}{ORANGE} lines represent original forecasting results, and \textcolor[rgb]{0, 0.5, 0}{GREEN} lines represent the forecasts after employing our approach. Besides, $ts120$ denotes this case is sampled at timestep-120 in the test set, and so on.}
    \subfigure[Illness-ts120]{
        \includegraphics[width=0.22\textwidth]{figure/Illness_Crossformer_120.pdf}
    }
    \subfigure[Illness-ts160]{
        \includegraphics[width=0.22\textwidth]{figure/Illness_Crossformer_160.pdf}
    }
    \subfigure[Traffic-ts200]{
        \includegraphics[width=0.22\textwidth]{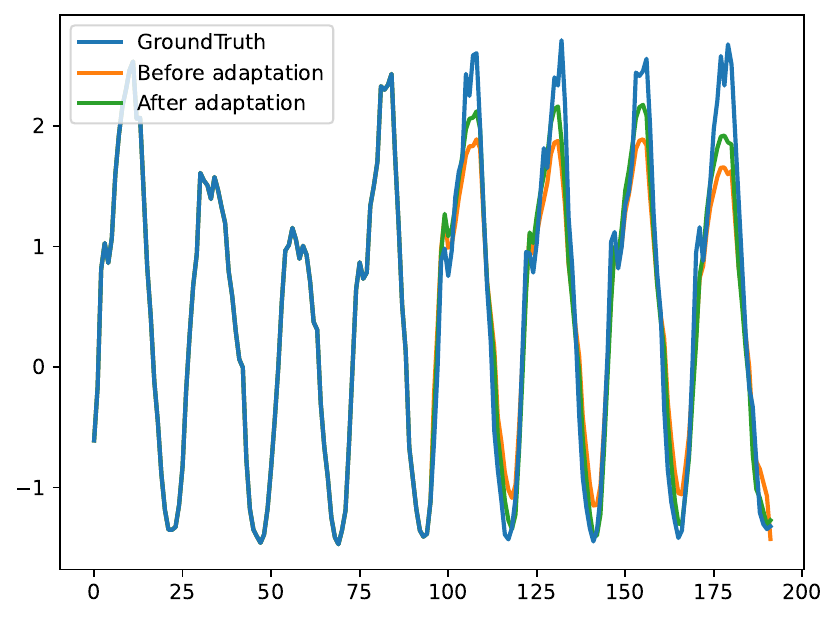}
    }
    \subfigure[Traffic-ts550]{
        \includegraphics[width=0.22\textwidth]{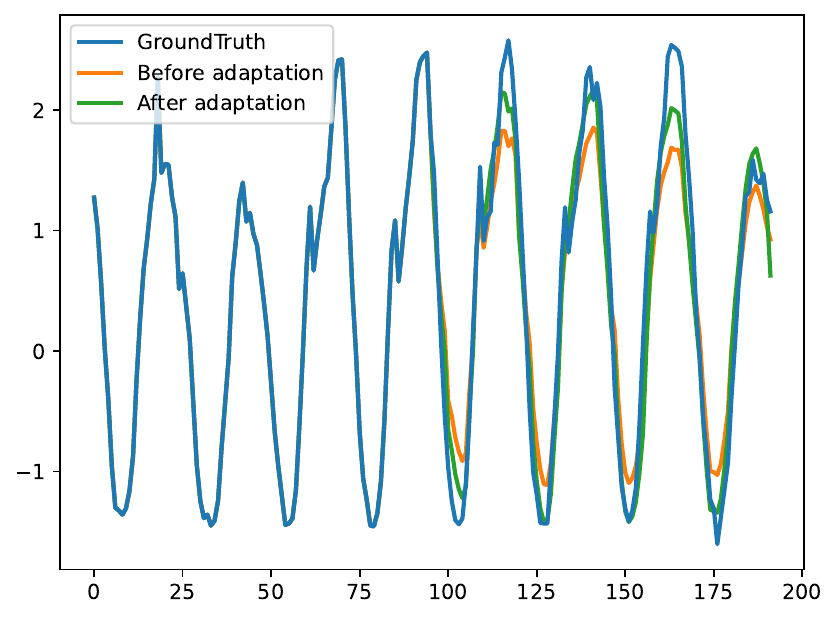}
    }
    \subfigure[Electricity-ts1500]{
        \includegraphics[width=0.22\textwidth]{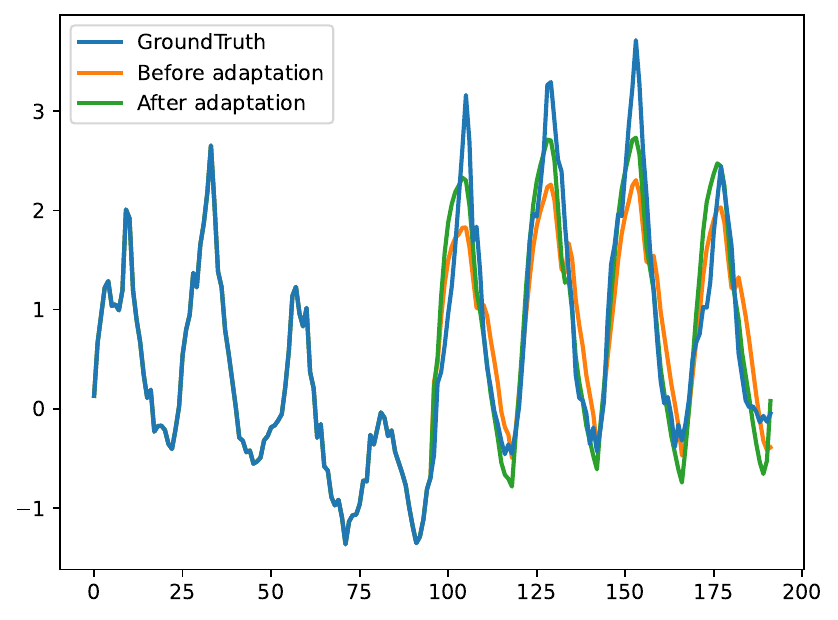}
    }
    \subfigure[Electricity-ts3000]{
        \includegraphics[width=0.22\textwidth]{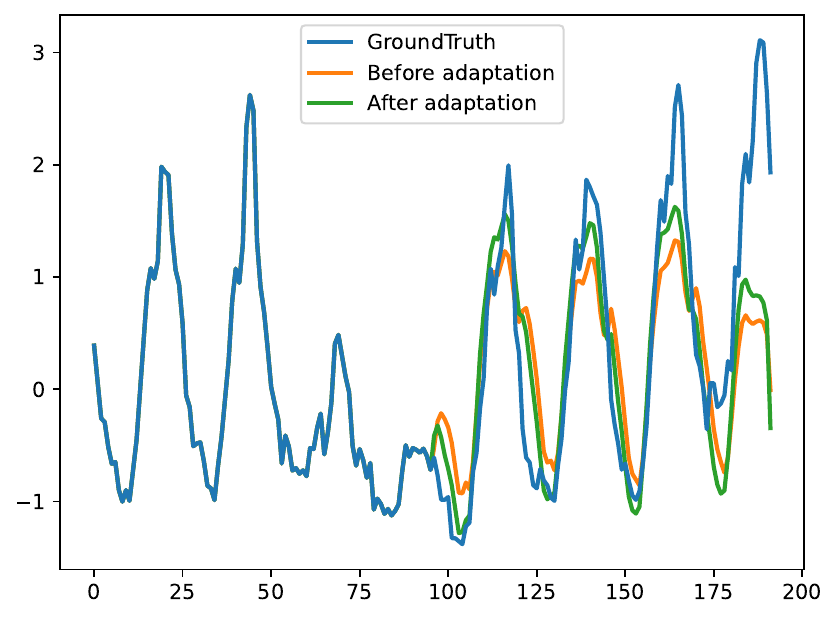}
    }
    \subfigure[ETTh1-ts250]{
        \includegraphics[width=0.22\textwidth]{figure/ETTh1_Informer_250.pdf}
    }
    \subfigure[ETTh1-ts650]{
        \includegraphics[width=0.22\textwidth]{figure/ETTh1_Informer_650.pdf}
    }
    \label{fig:visualization-of-prediction}
    \vspace{-0.3cm}
\end{figure*}